\documentclass{ecai}
\usepackage{graphicx}
\usepackage{latexsym}

\usepackage{times}
\usepackage{soul}
\usepackage{url}
\usepackage[hidelinks]{hyperref}
\usepackage[utf8]{inputenc}
\usepackage[small]{caption}
\usepackage{graphicx}
\usepackage{amsmath}
\usepackage{amsthm}
\usepackage{amssymb}
\usepackage{booktabs}
\usepackage{algorithm}
\usepackage{algorithmic}
\usepackage[switch]{lineno}
\usepackage[figuresright]{rotating}

\newtheorem{theorem}{Theorem}
\newtheorem{definition}{Definition}
\newtheorem{corollary}{Corollary}
\newtheorem{lemma}{Lemma}

\begin{document}

\begin{frontmatter}

\title{Uncertainty-driven Trajectory Truncation for Data Augmentation in Offline Reinforcement Learning}

\author[1]{\fnms{Junjie}~\snm{Zhang}\thanks{Equal contribution. Work done when Jiafei Lyu was an intern at Tencent.}}
\author[1;$*$]{\fnms{Jiafei}~\snm{Lyu}}
\author[2]{\fnms{Xiaoteng}~\snm{Ma}}
\author[2]{\fnms{Jiangpeng}~\snm{Yan}}
\author[2]{\fnms{Jun}~\snm{Yang}}
\author[3]{\fnms{Le}~\snm{Wan}}
\author[1]{\fnms{Xiu}~\snm{Li}\thanks{Corresponding Author. Email: li.xiu@sz.tsinghua.edu.cn.}}

\address[1]{Tsinghua Shenzhen International Graduate School, Tsinghua University}
\address[2]{Department of Automation, Tsinghua University}
\address[3]{IEG, Tencent}

\begin{abstract}
    Equipped with the trained environmental dynamics, model-based offline reinforcement learning (RL) algorithms can often successfully learn good policies from fixed-sized datasets, even some datasets with poor quality. Unfortunately, however, it can not be guaranteed that the generated samples from the trained dynamics model are reliable (e.g., some synthetic samples may lie outside of the support region of the static dataset). To address this issue, we propose \emph{\textbf{T}r\textbf{a}jectory \textbf{T}runcation with \textbf{U}ncertainty} (TATU), which adaptively truncates the synthetic trajectory if the accumulated uncertainty along the trajectory is too large. We theoretically show the performance bound of TATU to justify its benefits. To empirically show the advantages of TATU, we first combine it with two classical model-based offline RL algorithms, MOPO and COMBO. Furthermore, we integrate TATU with several off-the-shelf model-free offline RL algorithms, e.g., BCQ. Experimental results on the D4RL benchmark show that TATU significantly improves their performance, often by a large margin. Code is available \href{https://github.com/pipixiaqishi1/TATU}{here}
\end{abstract}
\end{frontmatter}

\section{Introduction}
Offline reinforcement learning (RL) \cite{Lange2012BatchRL} defines the task of learning the optimal policy from a previously gathered fixed-sized dataset by some unknown process. Offline RL eliminates the need for online data collection, which is an advantage over online RL since interacting with the environment can be expensive, or even dangerous, especially in real-world applications. The advances in offline RL raise the opportunity of scaling RL algorithms in domains like autonomous driving \cite{Sallab2017DeepRL}, healthcare \cite{Riachi2021ChallengesFR}, robotics \cite{Gu2016DeepRL}, etc.

In offline setting, the collected dataset often provides limited coverage in the state-action space. It is therefore hard for the offline agent to well-evaluate out-of-distribution (OOD) state-action pairs, due to the accumulation of extrapolation error \cite{Fujimoto2018OffPolicyDR} with bootstrapping. Many model-free offline RL algorithms have been proposed to avoid OOD actions, such as compelling the learned policy to stay close to the data-collecting policy (behavior policy) \cite{Wu2019BehaviorRO,Kumar2019StabilizingOQ,Fujimoto2018OffPolicyDR}, learning conservative value function \cite{Kumar2020ConservativeQF}, leveraging uncertainty quantification \cite{An2021UncertaintyBasedOR,Wu2021UncertaintyWA}, etc. Nevertheless, these methods often suffer from loss of generalization due to limited size of datasets \cite{Wang2021OfflineRL}.

\begin{figure}
    \centering
    \includegraphics[width=0.95\linewidth]{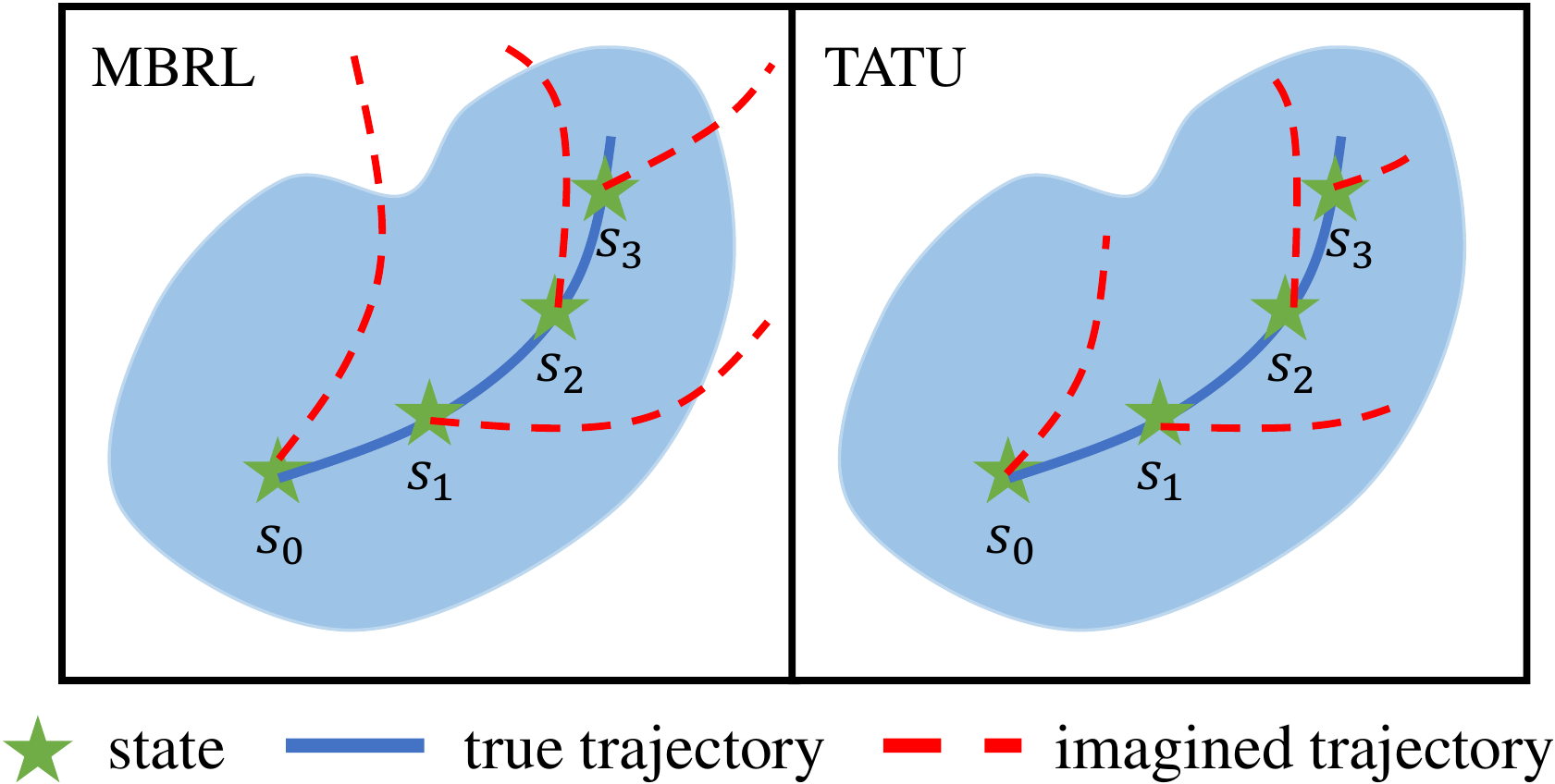}
    \caption{Comparison of TATU against common imagination generation in offline model-based RL (MBRL) methods. TATU truncates the imagined trajectory when the accumulated uncertainty is large and can therefore select good synthetic data in the trajectory. Blue region represents the support region of the static dataset.}
    \vspace{-0.2cm}
    \label{fig:illustration}
\end{figure}

Model-based RL (MBRL) methods, instead, improve the generalization of the agent by generating synthetic imaginations with the learned dynamics model. It has been revealed that directly applying model-based online RL methods like MBPO \cite{Janner2019WhenTT} fails on offline datasets \cite{Yu2020MOPOMO}. MBPO improves the sample efficiency for online model-based RL by introducing the branch rollout method that queries the environmental dynamics model for short rollouts. However, MBPO fails to resolve the issue of extrapolation error in the offline setting. Modern model-based offline RL methods usually leverage methods like uncertainty quantification \cite{Yu2020MOPOMO,Kidambi2020MOReLM}, penalizing value function to enforce conservatism \cite{Yu2021COMBOCO}, planning \cite{Zhan2021ModelBasedOP}, for learning meaningful policies from static logged data. Utilizing the learned dynamics model for offline data augmentation is also explored recently \cite{Wang2021OfflineRL,lyu2022double}. Nevertheless, there is no guarantee that the generated synthetic samples by the learned model are reliable and lie in the span of the offline dataset. \cite{lyu2022double} explores a double check mechanism for checking whether the generated samples are reliable, but it requires training bidirectional dynamics models.

In this paper, we propose trajectory truncation with uncertainty (TATU) for reliable imagined trajectory generation. With the estimation of accumulated uncertainty along the trajectory, the trajectory generation process becomes controllable and reliable and it's easy to adapt to different environments without much tuning of hyperparameters, e.g., rollout horizon. Our key idea is illustrated in Figure \ref{fig:illustration}. As shown, we resort to uncertainty-based methods and truncate the trajectory (i.e., the dynamics model will stop producing imaginations) if the accumulated uncertainty along this trajectory is approaching our tolerance range. By doing so, we \textit{enforce conservatism into synthetic data generation} by adaptively clipping the imagined trajectory and TATU can thus ensure that the dynamics model outputs reliable data. Generally, TATU involves three steps: (a) training the environmental dynamics model; (b) constructing an $\epsilon$-Pessimistic MDP (see Definition \ref{def:pmdp}); and (c) generating imagined trajectories with the $\epsilon$-Pessimistic MDP. We theoretically show that for \textit{any} policy, its performance in the true MDP is lower bounded by the performance in the $\epsilon$-Pessimistic MDP, which sheds light on applying TATU in practice. 

Intuitively, TATU can be combined with model-based offline RL methods like MOPO \cite{Yu2020MOPOMO}, COMBO \cite{Yu2021COMBOCO}, etc. Moreover, TATU can also benefit model-free offline RL methods by training an additional rollout policy and serving trustworthy offline data augmentation.

Our contributions can be summarized below:
\begin{itemize}
    \item We propose a novel uncertainty-based synthetic trajectory truncation method, TATU, which can be easily integrated into existing offline model-based and model-free RL algorithms.
    \item We theoretically justify the advantages of TATU.
    \item We combine TATU with two model-based offline RL algorithms, MOPO and COMBO, and observe significant performance improvement upon them.
    \item We further incorporate TATU with three model-free offline RL algorithms, TD3\_BC, CQL, and BCQ. Empirical results show that TATU markedly boosts their performance across a variety of D4RL benchmark tasks.
\end{itemize}

\section{Related Work}
\noindent\textbf{Model-free Offline RL.} Recent advances in model-free offline RL generally involve importance sampling \cite{Liu2019OffPolicyPG,Nachum2019AlgaeDICEPG}, injecting conservatism into value function \cite{Ma2021ConservativeOD,Cheng2022AdversariallyTA,Kostrikov2021OfflineRL,Kumar2020ConservativeQF,lyu2022mildly}, learning latent actions \cite{Zhou2020PLASLA,chen2022lapo}, forcing the learned policy to be close to the behavior policy \cite{fujimoto2021a,Fujimoto2018OffPolicyDR,Wu2019BehaviorRO,Kumar2019StabilizingOQ,Kostrikov2022OfflineRL}, adopting uncertainty measurement \cite{An2021UncertaintyBasedOR,Wu2021UncertaintyWA,Bai2022PessimisticBF}, using one-step method \cite{Brandfonbrener2021OfflineRW}, etc. Despite these advances, the generalization capability of the model-free offline RL algorithms is often limited due to the partial coverage of the dataset in the entire state-action space \cite{Wang2021OfflineRL,lyu2022double}.

\noindent\textbf{Model-based Offline RL.} Model-based offline RL methods focus on policy learning under a learned dynamics model. To mitigate extrapolation error \cite{Fujimoto2018OffPolicyDR}, several methods do not allow the learned policy to visit regions where the discrepancy between the learned dynamics and the true dynamics is large \cite{Argenson2020ModelBasedOP,Kidambi2020MOReLM,Yu2020MOPOMO,Yang2022ParetoPP}. Some researchers seek to constrain the learned policy to be similar to the behavior policy \cite{Cang2021BehavioralPA,Matsushima2020DeploymentEfficientRL,Swazinna2020OvercomingMB}. Also, sequential modeling is explored in offline RL \cite{Chen2021DecisionTR,Janner2021ReinforcementLA}. Some recent advances propose to leverage model-generated synthetic data for better training \cite{Yu2021COMBOCO,Wang2021OfflineRL,lyu2022double}, or to enhance model training \cite{Lu2021RevisitingDC,Lee2021RepresentationBO,Yang2022AUF,Rigter2022RAMBORLRA,Guo2022ModelBasedOR}, etc. We, instead, aim at offering higher-quality synthetic data for offline training via an uncertainty-based trajectory truncation method. The most relevant to our work are M2AC \cite{Pan2020TrustTM}, MOReL \cite{Kidambi2020MOReLM}, and MOPP \cite{Zhan2021ModelBasedOP}. MOReL requires the generated single sample (instead of the trajectory) at each step to lie in a safe region. MOPP leverages offline planning and sorts the \textit{entire} trajectory for high rewarding transitions based on uncertainty after planning is done. Both MOReL and MOPP optimize policy in a model-based manner. M2AC is an online MBRL algorithm that rejects the transitions (or trajectories) based on the instance uncertainty and modifies the way of measuring target value. TATU, instead, only serve data augmentation and its imagination generation process can be isolated from policy optimization. TATU truncates the synthetic trajectory based on the \emph{accumulated} uncertainty, and adds penalty to the generated samples. TATU \textit{early stops} the imagined trajectory and puts the whole truncated imagined trajectory into the model buffer for policy optimization. 


\section{Background}
\label{sec:background}
Reinforcement learning (RL) problems can be formulated by a Markov Decision Process (MDP), which consistes of $\mathcal{M}=\langle\mathcal{S},\mathcal{A}, r, P, \rho_0 ,\gamma\rangle$, where $\mathcal{S}$ denotes the state space, $\mathcal{A}$ is the action space, $r:\mathcal{S} \times \mathcal{A} \rightarrow \mathbb{R}$ is the scalar reward function, $P(s^\prime|s,a)$ is the transition probability, $\rho_0$ is the initial state distribution, $\gamma\in[0,1)$ is the discount factor. The goal of RL is to find a policy $\pi: s\rightarrow \Delta(\mathcal{A})$ such that the expected discounted cumulative long-term rewards with states sampled according to $\rho_0$ is maximized:
\begin{equation}
    \label{eq:return}
    \max_\pi J_{\rho_0}(\pi,\mathcal{M}) := \mathbb{E}_{s\sim\rho_0,\pi}\left[ \sum_{t=0}^\infty \gamma^t r(s_t,a_t) \bigg|s_0 = s \right],
\end{equation}
\noindent where $\Delta(\cdot)$ denotes the probability simplex. We further denote the optimal policy $\pi^* := \arg\max_\pi J_{\rho_0}(\pi,\mathcal{M})$.

In offline RL, the agent can not interact with the environment, and can only get access to a previously logged dataset $\mathcal{D}_{\rm env}=\{(s,a,r,s^\prime)\}$. The data can be collected by one or a mixture of behavior policy $\mu$, which are often unknown. Offline RL aims at learning the optimal batch-constraint policy. Model-based offline RL methods learn the environmental dynamics model $\hat{P}(\cdot|s,a)$ solely from the dataset, and leverage the learned model to optimize the policy. Typically, the transition probability model $\hat{P}(\cdot|s,a)$ is trained using maximum likelihood estimation. A reward function $\hat{r}(s,a)$ can also be learned if it is unknown. We assume the reward function is known in this work. Then we can construct a model MDP $\hat{\mathcal{M}}=\langle \mathcal{S},\mathcal{A}, r,\hat{P},\hat{\rho}_0,\gamma\rangle$. During training, the synthetic samples generated by the learned dynamics model will be put into the model buffer $\mathcal{D}_{\rm model}$ and the policy is finally updated using data sampled from $\mathcal{D}_{\rm env}\cup \mathcal{D}_{\rm model}$.

\section{Trajectory Truncation with Uncertainty}
In this section, we first provide some theoretical insights of our proposed trajectory truncation with uncertainty (TATU). Then, we present the detailed practical algorithm for TATU.
\subsection{Theoretical Analysis}
To determine whether the generated synthetic trajectory is reliable, we define the accumulated uncertainty along the trajectory (AUT) and truncation indicator below.
\begin{definition}
\label{def:usad}
Given a state-action sequence $\{s_i, a_i\}_{i=0}^h$, define accumulated uncertainty along the trajectory at step $t$ as:
\begin{equation}
    U_t(s_t,a_t) = \sum_{i=0}^t \gamma^t D_{\mathrm{TV}}(\hat{P}(\cdot|s_i,a_i),P(\cdot|s_i,a_i)).
\end{equation}
Furthermore, we define the truncation indicator as:
\begin{equation}
    T^\epsilon_t(s,a) = \begin{cases} 0, \quad \mathrm{if}\, U_t(s_t,a_t)\le\epsilon, \\ 1, \quad \mathrm{otherwise,}  \end{cases}
\end{equation}
\end{definition}
\noindent where $\hat{P}(\cdot|s,a)$ is the trained environmental transition probability, $P(\cdot|s,a)$ is the true dynamics, and $D_{\rm TV}(p,q)$ is the total variation distance between two distributions $p,q$. Intuitively, $U_t(s,a)$ measures how much uncertainty is accumulated along the imagined trajectory and $T^\epsilon_t(s,a)$ illustrates whether the trajectory ought to be truncated or not. If the accumulated uncertainty is too large (e.g., exceeds the threshold $\epsilon$), then the synthetic trajectory will be truncated and no more forward imagination will be included in the model buffer. We denote $u(s,a) = D_{\mathrm{TV}}(\hat{P}(\cdot|s,a),P(\cdot|s,a))$ as the uncertainty measurement for state-action pair $(s,a)$. As a valid uncertainty measurement, $u(s,a)$ ought to satisfy $0\le u(s,a) \le u_0<\infty$ for some positive constant $u_0$. Obviously, our defined uncertainty measurement above meets this requirement since the total variation distance is bounded. We denote the maximum uncertainty is $u_{\rm max}$, then we have $|u(s,a)|\le u_{\rm max}\le 1<\infty,\forall\, s,a$.  It is easy to find that AUT $U_t(s_t,a_t) = \sum_{i=0}^t \gamma^t u(s_i,a_i)$.

We further define the $\epsilon$-Pessimistic MDP as follows:

\begin{definition}[$\epsilon$-Pessimistic MDP]
\label{def:pmdp}
The $\epsilon$-Pessimistic MDP is defined by the tuple $\hat{\mathcal{M}}_p = \langle \mathcal{S},\mathcal{A},r_p, \hat{P}_p, \hat{\rho}_0, \gamma \rangle$, where $\mathcal{S}, \mathcal{A}$ are state space and action space in the actual MDP $\mathcal{M}$, $\hat{\rho}_0$ is the initial state distribution learned from the dataset $\mathcal{D}$. The transition dynamics and reward function for $\hat{\mathcal{M}}_p$ gives:
\begin{equation}
    \label{eq:pmdpprob}
    \hat{P}_p(\cdot|s_t,a_t)=\begin{cases} 0, \qquad \qquad \mathrm{if}\, T^\epsilon_t(s_t,a_t)=1, \\
    \hat{P}(\cdot|s_t,a_t),\qquad \mathrm{otherwise,}  \end{cases}
\end{equation}
\begin{equation}
    \label{eq:pmdpreward}
    r_p(s_t,a_t) = \begin{cases} r(s_t,a_t) - \lambda u(s_t,a_t)-\kappa, \mathrm{if}\, T^\epsilon_t(s_t,a_t)=1, \\ 
    r(s_t,a_t)-\lambda u(s_t,a_t),\qquad\mathrm{otherwise,} \end{cases}
\end{equation}
\end{definition}
\noindent where $\kappa$ is the penalty for truncation state, and $\lambda$ is the uncertainty penalty coefficient for each state-action pair $(s,a)$. Note that the summation of $\hat{P}_p$ is not 1, while it is valid because for a fixed offline dataset, the cases of $T_t^\epsilon(s_t,a_t)=1$ are almost sure. Hence, we can rearrange the probability distribution of $\hat{P}_p(\cdot|s_t,a_t)$ by weighting with a constant. By introducing the $\epsilon$-Pessimistic MDP formulation, we involve conservatism into the reward function and the imagined state-action sequence. The trajectory imagination will be terminated if the truncation indicator $T_t^\epsilon=1$ i.e., the accumulated uncertainty $U_t(s_t,a_t)$ exceeds $\epsilon$ at horizon $t$. An additional reward penalty $\kappa$ will be allocated to the termination state in the synthetic trajectory.

Assume that the reward function is bounded almost surely, i.e., $|r(s,a)|\le r_{\rm max},\forall\, s,a$. We now derive the performance bound for our $\epsilon$-Pessimistic MDP. Due to the space limit, all proofs are omitted and deferred to Appendix \ref{sec:missingproofs}.
\begin{theorem}
\label{theo:performancebound}
Denote $\bar{r}=r_{\rm max}+\lambda u_{\rm max}+\kappa$. Then the return of any policy $\pi$ in the $\epsilon$-Pessimistic MDP $\hat{\mathcal{M}}_p$ and its original MDP $\mathcal{M}$ satisfies:
\begin{equation}
    \begin{aligned}
        J_{\hat{\rho}_0}(\pi,\hat{\mathcal{M}}_p) \ge &J_{\rho_0}(\pi,\mathcal{M}) - \dfrac{2\bar{r}}{1-\gamma}\cdot D_{\rm TV}(\rho_0,\hat{\rho}_0) \\
        & - \bar{r} \cdot \epsilon - \dfrac{\bar{r}}{1-\gamma},
    \end{aligned}
\end{equation}
\begin{equation}
    \begin{aligned}
    J_{\hat{\rho}_0}(\pi,\hat{\mathcal{M}}_p) \le &J_{\rho_0}(\pi,\mathcal{M}) + \dfrac{2\bar{r}}{1-\gamma}\cdot D_{\rm TV}(\rho_0,\hat{\rho}_0) + \bar{r}\cdot \epsilon.
    \end{aligned}
\end{equation}
\end{theorem}

\noindent\textbf{Remark:} Our method incurs a tighter performance bound compared to MOReL. To be specific, our methods guarantees $|J_{\hat{\rho}_0}(\pi,\hat{\mathcal{M}}_p)-J_{\rho_0}(\pi,\mathcal{M})| \le \mathcal{O}\left(\frac{r_{\rm max}}{1-\gamma}\right)$ thanks to the uncertainty-based trajectory truncation. However, MOReL only ensures that $|J_{\hat{\rho}_0}(\pi,\hat{\mathcal{M}}_p)-J_{\rho_0}(\pi,\mathcal{M})| \le \mathcal{O}\left(\frac{r_{\rm max}}{(1-\gamma)^2}\right)$.

Theorem \ref{theo:performancebound} indicates that the return difference for any policy $\pi$ in the true MDP and the $\epsilon$-Pessimistic MDP relies on mainly three parts: (1) the total variation distance between the learned and true initial state distribution $D_{\rm TV}(\rho_0,\hat{\rho}_0)$; (2) the threshold for accumulated uncertainty $\epsilon$; (3) the upper bound for the pessimistic reward function (since $|r_p|\le \bar{r}$ almost surely). As an immediate corollary, we have
\begin{corollary}
\label{coro:suboptimal}
If the policy in the $\epsilon$-Pessimistic MDP is $\delta_\pi$ sub-optimal, i.e., $J_{\hat{\rho}_0}(\pi,\hat{\mathcal{M}}_p)\ge J_{\hat{\rho}_0}(\pi^*,\hat{\mathcal{M}}_p) - \delta_\pi$, then we have
\begin{equation*}
    \begin{aligned}
    J_{\rho_0}(\pi^*,\mathcal{M}) - J_{\rho_0}(\pi,\mathcal{M}) \le &\delta_\pi + \dfrac{4\bar{r}}{1-\gamma}\cdot D_{\rm TV}(\rho_0,\hat{\rho}_0) \\
    & +2\bar{r} \epsilon + \dfrac{\bar{r}}{1-\gamma}.
\end{aligned}
\end{equation*}
\end{corollary}
This corollary illustrates that the sub-optimality ($J(\pi^*,\mathcal{M}) - J(\pi,\mathcal{M})$) of the policy in the true MDP is bounded by the sub-optimality of the policy trained with the $\epsilon$-Pessimistic MDP. If the sub-optimality of the policy learned in the pessimistic MDP is small (i.e., $\delta_\pi$ is small), then the sub-optimality of the policy in the true MDP will also be small. Furthermore, if the dataset covers all possible transitions, i.e., the dataset is large enough, then it is easy to find both $D_{\rm TV}(\rho_0,\hat{\rho}_0)$ and $\epsilon$ approach 0 since all of the imagined samples will lie in the support region of the dataset with high probability. Naturally, the above upper bound thus can be further simplified.
\begin{corollary}
    \label{coro:simplified}
    If the dataset is large enough, then we have
    \begin{equation*}
    \begin{aligned}
    J_{\rho_0}(\pi^*,\mathcal{M}) - J_{\rho_0}(\pi,\mathcal{M}) \le \delta_\pi + \dfrac{\bar{r}}{1-\gamma}.
    \end{aligned}
    \end{equation*}
\end{corollary}

Corollary \ref{coro:simplified} also reveals that even when the dataset is large, the performance deviation between the optimal policy and the learnt policy is determined by the sub-optimality in the pessimistic MDP.

We then present the algorithmic details for TATU below.

\subsection{Practical Algorithm}
The practical implementation of TATU can be generally divided into three steps:

\noindent\textbf{Step 1: Training Dynamics Models:} Following prior work \cite{Janner2019WhenTT}, we train the dynamics model $\hat{P}(\cdot|s,a)$ with a neural network $\hat{p}_\psi(s^\prime|s,a)$ parameterized by $\psi$ that produces a Gaussian distribution over the next state, i.e., $\hat{p}_\psi(s^\prime|s,a)=\mathcal{N}(\mu_\theta(s,a), \Sigma_\phi(s,a)), \psi = \{\theta,\phi\}$. We train an ensemble of $N$ dynamics models $\{\hat{p}_{\psi}^i=\mathcal{N}(\mu_\theta^i,\Sigma_\phi^i)\}_{i=1}^N$. We denote the loss function for training the forward dynamics model as $\mathcal{L}_\psi$. Then each model in the ensemble is trained independently via maximum log-likelihood:
\begin{equation}
    \label{eq:likelihood}
    \mathcal{L}_\psi = \mathbb{E}_{(s,a,r,s^\prime)\sim\mathcal{D}}\left[ -\log \hat{p}_\psi(s^\prime|s,a) \right].
\end{equation}
We model the difference in the current state and next state, i.e., $\mu_\theta(s,a) = s + \delta_\theta(s,a)$ to ensure local continuity.

\noindent\textbf{Step 2: Constructing $\epsilon$-Pessimistic MDP:} After the environmental dynamics model is well-trained, we use it to construct the $\epsilon$-Pessimistic MDP by following Equation (\ref{eq:pmdpprob}) and (\ref{eq:pmdpreward}) in Definition \ref{def:pmdp}. It is then important to decide which uncertainty measurement to use in practice (we cannot use $D_{\rm TV}(\hat{P}(\cdot|s,a),P(\cdot|s,a))$ as true transition probability $P(s^\prime|s,a)$ is often unknown and inaccessible), and how to decide the uncertainty threshold $\epsilon$. As a valid and reasonable uncertainty measurement, we require $u(s,a)$ can capture how uncertain the state-action pair $(s,a)$ is, and satisfy $|u(s,a)|\le u_{\rm max} < \infty$ almost surely for any $(s,a)$. MOReL \cite{Kidambi2020MOReLM} adopts the ensemble discrepancy as the uncertainty quantifier, i.e., $u(s,a) = \max_{i,j}\|\mu^i_\theta(s,a) - \mu^j_\theta(s,a)\|_2$, where $\|\cdot\|_2$ is the L2-norm and $\mu^i_\theta(s,a), \mu_\theta^j(s,a), i,j\in\{1,\ldots,N\}$ are mean vectors of the Gaussian distributions in the ensemble. MOPO \cite{Yu2020MOPOMO} utilizes the maximum standard deviation of the learned models in the ensemble as the uncertainty estimator, i.e., $u(s,a) = \max_{i=1}^N\|\Sigma_\phi^i(s,a)\|_F$, where $\|\cdot\|_F$ is the Frobenius-norm. $\|{\bf A}\|_F = \sqrt{\sum_{i=1}^m\sum_{j=1}^n|a_{ij}|}$ for matrix ${\bf A}$ with size $m\times n$. Empirically, we do not observe much performance difference with these quantifiers in our experiments. We then use an MOPO-style uncertainty estimator in this work. Such uncertainty measurement generally reveals whether the state-action pair is reliable, and satisfies $|u(s,a)|\le u_{\rm max}<\infty$. One can also enforce this by clipping $u(s,a)$ to $[-u_{\rm max}, u_{\rm max}]$ with a manually set upper bound.

As for the uncertainty threshold $\epsilon$, one na\"ive way is to set a constant threshold. However, this often lacks flexibility and the best threshold may be quite varied for different datasets. We instead propose to measure the uncertainty on \textit{all of the samples} in the dataset $\mathcal{D}$. We then take the maximum transition uncertainty in the dataset and set the uncertainty threshold based on:
\begin{equation}
    \label{eq:uncertaintythreshold}
    \epsilon = \dfrac{1}{\alpha} \max_{i\in[|\mathcal{D}|]} u(s_i,a_i)=\dfrac{1}{\alpha} \max_{i\in[|\mathcal{D}|]} \max_{j\in[N]} \|\Sigma^j_\phi(s_i,a_i)\|_F,
\end{equation}
\noindent where $[k] = \{1,\ldots,k\}$, $(s_i,a_i)\sim\mathcal{D},i\in[|\mathcal{D}|]$, and $\alpha\in\mathbb{R}_+$ is the key hyperparameter that controls the strength of the threshold. By using a larger $\alpha$, TATU becomes more conservative and includes fewer imagined samples into the model buffer $\mathcal{D}_{\rm model}$. While if a small $\alpha$ is used, TATU exhibits more tolerance to generated data. The total uncertainty is controlled by the maximum uncertainty in the real dataset (i.e., $\max_{i=1}^{|\mathcal{D}|}u(s_i,a_i)$), which usually varies with different datasets. This allows a flexible and reasonable uncertainty threshold for adaptively selecting good imaginations.

\noindent\textbf{Step 3: Conservative Trajectory Generation:} Once the uncertainty threshold $\epsilon$ is ready, we can generate conservative trajectories. We first sample starting state $s_0$ from the static dataset $\mathcal{D}$ and generate an imagined trajectory $\hat{\tau} = \{\hat{s}_j,a_j,r_j,\hat{s}_{j+1}\}_{j=0}^{h-1}$ with the learned dynamics $\hat{p}_\psi$, where $h$ is the horizon length. We simultaneously calculate the accumulated uncertainty of the imagined trajectory $U_t(s_t,a_t)$ at horizon $t$, $t\in[1,h]$. The generated trajectory is truncated at horizon $t$ if $U_t(s_t,a_t)>\epsilon$. In deep RL, we usually sample a mini-batch of data from the real dataset $\mathcal{D}$ of size $B$ with bootstrapping. By leveraging the learned dynamics model, we then can get $B$ imagined trajectories, provably with different lengths (as the accumulated uncertainty varies with different starting states). We can incorporate TATU with existing popular model-based offline RL algorithms such as MOPO \cite{Yu2020MOPOMO}, COMBO \cite{Yu2021COMBOCO}, etc. Here, the actions in the imagined trajectory are generated with the learned policy $\pi$.

\begin{algorithm}[tb]
\caption{Trajectory Truncation with Uncertainty (TATU)}
\label{alg:algtatumodelbased}
\begin{algorithmic}[1] 
\STATE \textbf{Require:} Offline dataset $\mathcal{D}$, number of iterations $N$, horizon $h$, reward penalty coefficient $\lambda$, termination penalty $\kappa$
\STATE Initialize model buffer $\mathcal{D}_{\rm model} \leftarrow \emptyset$
\STATE Train the ensemble dynamics models $\{\hat{p}^i_\psi(s^\prime|s,a) = \mathcal{N}(\mu_\theta^i(s,a), \Sigma_\phi^i(s,a))\}_{i=1}^N$ on $\mathcal{D}$ using Equation (\ref{eq:likelihood})
\STATE Calculate the truncation threshold $\epsilon$ using Equation (\ref{eq:uncertaintythreshold})
\STATE (Optional) Train a rollout policy via Equation (\ref{eq:cvae})
\FOR{epoch from 1 to $N$}
\STATE Sample state $s_0$ from dataset $\mathcal{D}$
\FOR{$j$ in 1 to $h$}
\STATE Sample an action $a_j\sim\pi(\cdot|s_j)$ // learned policy
\STATE (Optional) Draw an action $a_j$ from the rollout policy
\STATE Randomly pick dynamics $\hat{p}$ from $\{\hat{p}_\psi^i\}_{i=1}^N$ and sample next state $s_{j+1}\sim\hat{p}$
\STATE Calculate reward according to Equation (\ref{eq:pmdpreward})
\STATE Calculate accumulated uncertainty along the trajectory $U_j = \sum_{k=1}^j \left[\max_{i=1}^N \|\Sigma_\phi^i(s_k,a_k)\|_F\right]$
\IF{$U_j \le \epsilon$}
\STATE Put the imagined transition $(s_j,a_j,r_j,s_{j+1})$ into the model buffer $\mathcal{D}_{\rm model}$
\ELSE
\STATE break \qquad // Truncate the synthetic trajectory
\ENDIF
\ENDFOR
\STATE Sample data from $\mathcal{D}\cup\mathcal{D}_{\rm model}$ and use the base algorithm (e.g., SAC, CQL, BCQ) to optimize policy $\pi$
\ENDFOR
\end{algorithmic}
\end{algorithm}

\begin{table*}
\caption{Normalized average score comparison of TATU+MOPO and TATU+COMBO against their base algorithms and some recent baselines on the D4RL MuJoCo ``-v2" dataset. half = halfcheetah, r = random, m = medium, m-r = medium-replay, m-e = medium-expert. Each algorithm is run for 1M gradient steps with 5 different random seeds. We report the final performance. $\pm$ captures the standard deviation. We bold the top 2 score of the left part and the best score of the right part.}
\label{tab:resultsonmodelbased}
\centering
\begin{tabular}{ccccc|cccccc}
\toprule
Task Name                 & TATU+MOPO   & MOPO      & TATU+COMBO   & COMBO & BC & CQL & IQL & DT & MOReL \\ 
\midrule
half-r        & 33.3$\pm$2.6  & \textbf{35.9} & 29.3$\pm$2.7   & \textbf{38.8} & 2.2 & 17.5 & 13.1 & - &  \textbf{38.9} \\
hopper-r            & \textbf{31.3}$\pm$0.6  & 16.7 & \textbf{31.6}$\pm$0.6   & 17.9 & 3.7 & 7.9 & 7.9 & - & \textbf{38.1} \\
walker2d-r           & \textbf{10.4}$\pm$0.7  & 4.2  &    5.3$\pm$0.0          & \textbf{7.0}  & 1.3 & 5.1 & 5.4 & - &  \textbf{16.0}  \\
half-m-r & \textbf{67.2}$\pm$3.3  & \textbf{69.2} & 57.8$\pm$2.7   & 55.1 & 37.6 & \textbf{45.5} & 44.2 & 36.6 & 44.5 \\
hopper-m-r      & \textbf{104.4}$\pm$0.9 & 32.7 & \textbf{100.7}$\pm$1.3   & 89.5 & 16.6 & 88.7 & \textbf{94.7} & 82.7 & 81.8 \\
walker2d-m-r    & \textbf{75.3}$\pm$0.2  & 73.7 & \textbf{75.3}$\pm$1.7   & 56.0 & 20.3 & \textbf{81.8} & 73.8 & 66.6 & 40.8 \\
half-m        & 61.9$\pm$2.9  & \textbf{73.1} & \textbf{69.2}$\pm$2.8   & 54.2 & 43.2 & 47.0 & 47.4 & 42.6 & \textbf{60.7} \\
hopper-m             & \textbf{104.3}$\pm$1.3 & 38.3 & \textbf{100.0}$\pm$1.3  & 97.2 & 54.1 & 53.0 & 66.2 & 67.6 & \textbf{84.0} \\
walker2d-m           & \textbf{77.9}$\pm$1.6  & 41.2 & 77.4$\pm$0.9   & \textbf{81.9}  & 70.9 & 73.3 & \textbf{78.3}  & 74.0 & 72.8 \\
half-m-e & 74.1$\pm$1.4  & 70.3 & \textbf{96.4}$\pm$3.6   & \textbf{90.0}  & 44.0 & 75.6 & 86.7 & \textbf{86.8} & 80.4 \\
hopper-m-e      & \textbf{107.0}$\pm$1.3 & 60.6 & 106.5$\pm$0.4 & \textbf{111.1} & 53.9 & 105.6 & 91.5 & \textbf{107.6} & 105.6 \\
walker2d-m-e    & \textbf{107.9}$\pm$0.9 & 77.4 & \textbf{114.6}$\pm$0.7  & 103.3 & 90.1 & 107.9 & \textbf{109.6} & 108.1 & 107.5 \\
\midrule
Average score & \textbf{71.3} & 49.4 & \textbf{72.0} & 66.8 & 36.5 & 59.1 & 59.9 & - & \textbf{64.3} \\
\bottomrule
\end{tabular}
\end{table*}

TATU can also be incorporated with off-the-shelf model-free offline RL methods, where TATU serves as a role of offline data augmentation to improve the generalization ability of model-free offline RL algorithms. At this time, we need to train an additional rollout policy such that the data augmentation process is isolated from the policy optimization process, i.e., the dataset is augmented beforehand. To avoid potential OOD actions, we model the rollout policy using the conditional variational auto-encoder (CVAE) \cite{Kingma2013AutoEncodingVB}. We choose CVAE as it guarantees that the generated actions lie in the span of the dataset\footnote{One can also use other generative models like GAN \cite{Mirza2014ConditionalGA}, diffusion model \cite{Ho2020DenoisingDP}, etc.}. The CVAE is made up of an encoder $E_\xi(s,a)$ that produces a latent variable $z$ under the Gaussian distribution and a decoder $D_\nu(s,z)$ that maps $z$ to the desired space. The objective function for training CVAE is shown below.
\begin{equation}
    \label{eq:cvae}
    \begin{aligned}
    \mathcal{L}_{\rm CVAE} = \mathop{\mathbb{E}}\limits_{\substack{(s,a,r,s^\prime)\sim\mathcal{D}, z\sim{E}_{\xi}(s,a)}} [ & \left( a - {D}_{\nu}(s,z) \right)^2 \\
    & + \left . D_{\mathrm{KL}}\left( {E}_{\xi}(s,a) \| \mathcal{N}(0,\bf{I}) \right) \right],
    \end{aligned}
\end{equation}
\noindent where the encoder $E_\xi(s,a)$ and decoder $D_\nu(s,z)$ are parameterized by $\xi,\nu$, respectively. $D_{\rm KL}(p \| q)$ denotes the KL-divergence between two distributions $p,q$, and $\bf I$ is an identity matrix. When drawing an action from the rollout policy, we first sample a latent variable $z$ from the multivariate Gaussian distribution $\mathcal{N}(0, {\bf I})$ and process it along with the current state $s$ with the decoder $D_\nu(s,z)$ to get the resulting action.

We summarize the detailed pseudo code for TATU in Algorithm \ref{alg:algtatumodelbased}. Note that when drawing samples from $\mathcal{D}\cup\mathcal{D}_{\rm model}$, we sample a proportion of $\eta B$ real data from real dataset $\mathcal{D}$ and a proportion of $(1-\eta)B$ synthetic data from model buffer $\mathcal{D}$ given the batch size $B$ and the real data ratio $\eta\in[0,1]$. We then use the underlying algorithms (e.g., CQL) to optimize the policy. To accommodate the reproducibility, we include our source code in the supplementary material and will open-source the code upon acceptance.

\section{Experiments}
In this section, we empirically examine how much can TATU benefit existing offline RL methods. Throughout our experimental evaluation, we aim at answering the following questions: (1) How much performance gain can model-based offline RL methods acquire when combined with TATU? (2) Can TATU also benefit model-free offline RL algorithms?

To answer these questions, we first combine TATU with two popular model-based offline RL algorithms, MOPO \cite{Yu2020MOPOMO} and COMBO \cite{Yu2021COMBOCO} to examine whether TATU can boost their performance in Section \ref{sec:combinemodelbased}. We also integrate TATU with three off-the-shelf model-free offline RL methods including CQL \cite{Kumar2020ConservativeQF}, TD3\_BC \cite{fujimoto2021a}, and BCQ \cite{Fujimoto2018OffPolicyDR}, to see whether TATU can also benefit them in Section \ref{sec:combinemodelfree}. We conduct experiments on D4RL \cite{Fu2020D4RLDF} MuJoCo datasets for evaluation. In Section \ref{sec:parameterstudy}, We provide a detailed parameter study on some important hyperparameters in TATU, e.g., rollout horizon. We further compare TATU against other data selection methods in Section \ref{sec:compareothermethods}.

\subsection{Combination with Model-based Offline RL}
\label{sec:combinemodelbased}
Since TATU is designed intrinsically for reliable imagination generation using a learnt dynamics model, we incorporate it with two widely used offline model-based RL algorithms, MOPO \cite{Yu2020MOPOMO} and COMBO \cite{Yu2021COMBOCO}, giving rise to TATU+MOPO and TATU+COMBO. Since MOPO already penalizes the reward signal with uncertainty, we only add an additional penalty $\kappa$ to the termination state. TATU modifies the way of imagination generation and concretely rejects bad transition tuples in these methods.

To examine whether TATU can achieve performance improvement upon MOPO and COMBO, we conduct experiments on 12 D4RL \cite{Fu2020D4RLDF} MuJoCo datasets, which includes three tasks: \textit{halfcheetah}, \textit{hopper}, \textit{walker2d}. We adopt four types of datasets for each task: \textit{random}, \textit{medium}, \textit{medium-replay}, and \textit{medium-expert}, as they are typically utilized for performance evaluation in model-based offline RL. We compare TATU+MOPO, TATU+COMBO against their base algorithms. We also compare them against some common baselines in offline RL, such as behavior cloning (BC), IQL \cite{Kostrikov2022OfflineRL}, Decision Transformer (DT) \cite{Chen2021DecisionTR}. We take the results of IQL on medium-level datasets from its original paper and run with its official implementation\footnote{https://github.com/ikostrikov/implicit\_q\_learning} on random and expert datasets. The results of BC are acquired by our own implementation. For methods that were originally evaluated on ``-v0" datasets, we retrain them with the official implementations on ``-v2" datasets, and take the results of other baselines from their original papers. 

\begin{table*}
\caption{Normalized average score comparison of TATU+TD3\_BC, TATU+CQL and TATU+BCQ against their base algorithms and some recent baselines on the D4RL MuJoCo ``-v2" dataset. half = halfcheetah, r = random, m = medium, m-r = medium-replay, m-e = medium-expert, e = expert. Each algorithm is run for 1M gradient steps across 5 different random seeds and the final mean performance is reported. $\pm$ captures the standard deviation. We bold the top 3 score of the left part and the best score of the right part.}
\label{tab:resultsonmodelfree}
\centering
\begin{tabular}{ccccccc|ccc}
\toprule
Task Name                 & TATU+TD3\_BC             & TD3\_BC     & TATU+CQL                & CQL    &TATU+BCQ                 &BCQ  &  BC & IQL & DT \\
\midrule
half-r        & \textbf{12.1}$\pm$2.3  & 11.0        & \textbf{27.4}$\pm$2.6 & \textbf{17.5}   & 2.3$\pm$2.3             &2.2 & 2.2 & \textbf{13.1} & -   \\
hopper-r             & \textbf{31.6}$\pm$0.6  & 8.5         & \textbf{32.3}$\pm$0.7 & 7.9    & \textbf{10.3}$\pm$0.8   &7.8 & 3.7 & \textbf{7.9} & -  \\
walker2d-r           & \textbf{21.4}$\pm$0.0  & 1.6         & \textbf{23.0}$\pm$0.0 & \textbf{5.1}   & 3.4$\pm$0.4              &4.9 & 1.3 & \textbf{5.4} & -   \\
half-m-r & \textbf{45.9}$\pm$0.6  & 44.6                & \textbf{48.0}$\pm$0.7         & \textbf{45.5}  & 43.5$\pm$0.3    &42.2 & 37.6 & \textbf{44.2} & 36.6  \\
hopper-m-r      & 65.7$\pm$3.9  & 60.9        & \textbf{96.8}$\pm$2.6 & \textbf{88.7}  & \textbf{72.9}$\pm$0.4   &60.9 & 16.6 & \textbf{94.7} & 82.7  \\
walker2d-m-r    & \textbf{81.9}$\pm$2.7  & \textbf{81.8}        &\textbf{85.5}$\pm$1.2 & \textbf{81.8}   & 77.7$\pm$1.0   &57.0 & 20.3 & \textbf{73.8} & 66.6  \\
half-m        & \textbf{48.4}$\pm$2.7  & \textbf{48.3}          & 44.9$\pm$0.3                & 47.0  & \textbf{47.6}$\pm$0.2    &46.6 & 43.2 & \textbf{47.4} & 42.6  \\
hopper-m             & \textbf{62.0}$\pm$1.0  & 59.3        & \textbf{68.9}$\pm$0.9 & 53.0  & \textbf{71.0}$\pm$0.3   &59.4 & 54.1 & 66.2 & \textbf{67.6} \\
walker2d-m           & \textbf{84.3}$\pm$0.2  & \textbf{83.7}        &  65.7$\pm$0.5  & 73.3 & \textbf{80.5}$\pm$0.4  &71.8 & 70.9 & \textbf{78.3} & 74.0 \\
half-m-e & \textbf{97.1}$\pm$3.2  & 90.7        & 78.9$\pm$0.5 & 75.6   & \textbf{96.1}$\pm$0.2   &\textbf{95.4}  & 44.0 & 86.7 & \textbf{86.8} \\
hopper-m-e      & \textbf{113.0}$\pm$1.7 & 98.0        & \textbf{111.5}$\pm$1.0 & 105.6 & \textbf{108.2}$\pm$0.3 &106.9 & 53.9 & 91.5 & \textbf{107.6} \\
walker2d-m-e    & \textbf{110.9}$\pm$0.5 & 110.1       & \textbf{110.2}$\pm$0.1  & 107.9 & \textbf{111.7}$\pm$0.3 &107.7 & 90.1 & \textbf{109.6} & 108.1 \\
half-e        & \textbf{97.4}$\pm$0.4  & \textbf{96.7}        & 90.8$\pm$3.4   &  \textbf{96.3} & \textbf{96.3}$\pm$1.4  &89.9  & 91.8 & \textbf{95.0} & -  \\
hopper-e             & \textbf{111.8}$\pm$1.1 & \textbf{107.8}      & 106.8$\pm$0.9  & 96.5 & 103.9$\pm$0.6         &\textbf{109.0} & 107.7 & \textbf{109.4} & - \\
walker2d-e           & \textbf{110.0}$\pm$0.1         & \textbf{110.2}& 108.3$\pm$0.1 & 108.5 & \textbf{109.7}$\pm$0.3 &106.3 & 106.7 &  \textbf{109.9} & - \\
\midrule
Average score & \textbf{72.9} & 67.5 & \textbf{73.3} & 67.3 & \textbf{69.0} & 64.5 & 49.6 & \textbf{68.9} & - \\
\bottomrule
\end{tabular}
\end{table*}

All of the algorithms are run for 1M gradient steps. Table \ref{tab:resultsonmodelbased} reports the normalized score over the final 10 evaluations for each task and their average performance over 5 different random seeds. It can be found that TATU markedly boosts the performance of both MOPO and COMBO on most of the datasets, often by a large and significant margin. For example, TATU+MOPO achieves a mean score of \textbf{104.4} on hopper-medium-replay-v2 and \textbf{77.9} on walker2d-medium-v2, while MOPO only attains 32.7 and 41.2 on them, respectively. It is worth noting that MOPO and COMBO gain \textbf{44\%} and \textbf{7.8\%} improvement in average score, respectively, with the aid of TATU. We also observe a competitive or better performance of TATU+MOPO and TATU+COMBO on the evaluated datasets against baseline methods like IQL, MOReL, etc. TATU+COMBO has the best average score (\textbf{72.0}) across all of the datasets, TATU+MOPO achieves an average score of \textbf{71.3}. Whereas, the best baseline only achieves 66.8. The empirical evaluations indicate that TATU can significantly benefit the existing model-based offline RL methods.

\subsection{Combination with Model-free Offline RL}
\label{sec:combinemodelfree}
We now demonstrate that TATU can also benefit model-free offline RL algorithms. We first train the dynamics model and an additional CVAE rollout policy. During data generation, we truncate the synthetic trajectory if the accumulated uncertainty along it is large. Importantly, we leverage the rollout policy to produce actions in the imagined trajectory, thus isolating the data augmentation process from the policy optimization process. 

To illustrate the generality and effectiveness of TATU, we combine it with three popular model-free offline RL algorithms, CQL \cite{Kumar2020ConservativeQF}, BCQ \cite{Fujimoto2018OffPolicyDR}, and TD3\_BC \cite{fujimoto2021a}, yielding TATU+CQL, TATU+BCQ, and TATU+TD3\_BC algorithms. We extensively compare them with their base algorithms and three model-free offline RL baselines (BC, IQL, DT) on 15 D4RL datasets, with additional \textit{expert} datasets of three tasks compared to the experimental evaluation in Section \ref{sec:combinemodelbased}.

We summarize the experimental results in Table \ref{tab:resultsonmodelfree}. We find that TATU significantly improves the performance of the base algorithms on many datasets, often outperforming them by a remarkable margin, especially on many poor-quality datasets (e.g., random). As an example, TATU+TD3\_BC has a mean normalized score of \textbf{31.6} on hopper-random-v2, while TD3\_BC itself only achieves 8.5. The three base methods, TD3\_BC, CQL, BCQ, gains \textbf{8.0\%}, \textbf{8.9\%} and \textbf{7.0\%} performance improvement in average score, respectively. With the conservative and reliable transition generation by our novel synthetic trajectory truncation method, model-free offline RL methods can benefit from better generalization capability, resulting in better performance. Based on these experiments, we conclude that TATU is general and can widely benefit both model-based and model-free offline RL algorithms.

Note that due to space limit, we omit the standard deviation for baseline methods, and the full comparison results of Table \ref{tab:resultsonmodelbased} and Table \ref{tab:resultsonmodelfree} can be found in Appendix \ref{sec:fullcomparison}.

\subsection{Parameter Study}
\label{sec:parameterstudy}
There are three most critical hyperparameters in TATU, the rollout horizon $h$, the threshold parameter $\alpha$, and the real data ratio $\eta\in[0,1]$ in a sampled batch from $\mathcal{D}\cup\mathcal{D}_{\rm model}$.

\begin{table}
\caption{Comparison of TATU+MOPO and TATU+TD3\_BC with different horizon lengths $h$ on hopper-medium-replay-v2. The results are run for 1M gradient steps and averaged over the final 10 evaluations and 5 random seeds. Mean performance along with the standard deviation are reported.}
\label{tab:horizon}
\centering
\begin{tabular}{ccc}
\toprule
horizon $h$    & TATU+MOPO                  & TATU+TD3\_BC        \\ 
\midrule
1               & 26.7$\pm$4.4             & 57.9$\pm$1.9          \\
3               & 100.4$\pm$1.3             & 70.0$\pm$4.9             \\
5               & 104.4$\pm$0.9            & 65.7$\pm$3.9          \\
7               & 102.0$\pm$0.7            & 71.1$\pm$12.4          \\
10              & 103.7$\pm$1.6            & 71.0$\pm$16.8       \\
\bottomrule
\end{tabular}
\end{table}

\noindent\textbf{Rollout horizon $h$.} This parameter measures how far that we allow the imagined trajectory to branch. With a larger horizon length for the dynamics model, more diverse data can be included in the model buffer, which, however, also increases the risk of involving bad transitions. This issue can be alleviated by TATU as it adaptively truncates the imagined trajectory. A large horizon length $h$ can therefore be adopted. To see the influence of $h$, we conduct experiments on hopper-medium-replay-v2 with TATU+MOPO and TATU+TD3\_BC. We keep other parameters fixed and sweep $h$ across $\{1,3,5,7,10\}$. Figure \ref{fig:main_horizonhoppermediumreplay} shows the corresponding learning curves. Results in Table \ref{tab:horizon} indicates that TATU enjoys better performance with a larger horizon and fails with horizon $1$. This may be because the diversity of the generated samples is not enough with a too small horizon length. With additional experiments, we find the optimal horizon $h$ is influenced by the quality of the dataset (please see Section \ref{app:implementation} for more details). The fact is that for expert datasets (with narrow span), generated samples are more likely to be OOD at a certain horizon $h$. On medium-quality datasets with a larger horizon, algorithms can benefit from conservative and diversified imaginations of TATU to improve performance. In our main experiments, we set the rollout horizon $h=5$ by default.

\begin{figure}
    \centering
    \includegraphics[width=0.45\linewidth]{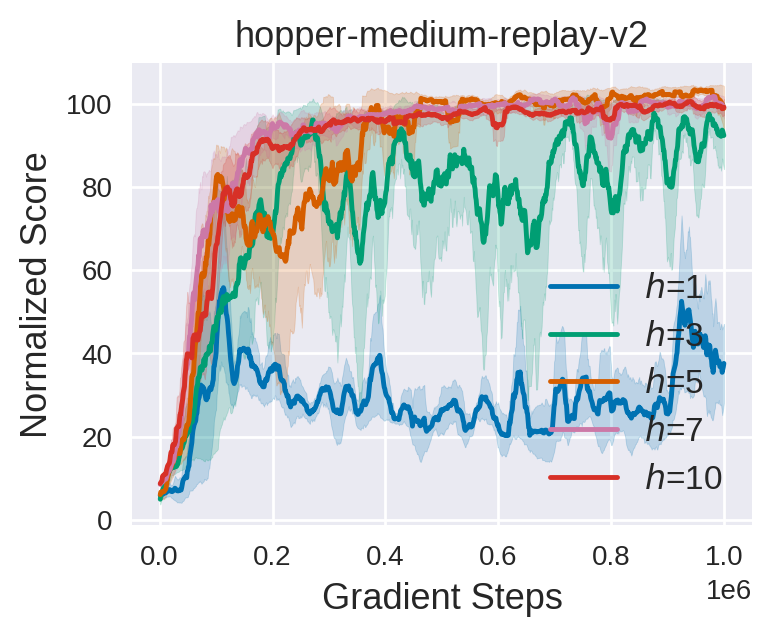}
    \includegraphics[width=0.45\linewidth]{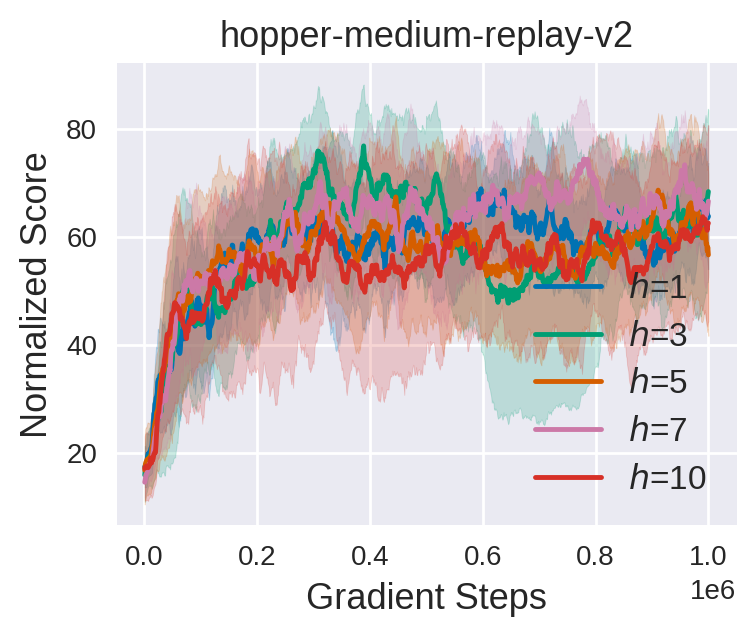}
    \caption{Performance of TATU+MOPO (left) and TATU+TD3\_BC (right) with different rollout horizons on hopper-medium-replay-v2. The results are averaged over 5 different random seeds, and we report the mean performance and the standard deviation.}
    \label{fig:main_horizonhoppermediumreplay}
\end{figure}

\noindent\textbf{Threshold coefficient $\alpha$.} The threshold parameter $\alpha$ is probably the most critical parameter for TATU, which controls the strength that we reject the generated imaginations (larger $\alpha$ will reject more transitions). This parameter is highly related to the quality of the dataset, e.g., a large $\alpha$ is better for expert-level datasets while $\alpha$ can be small for dataset that is collected by a random policy. Generally, we require $\alpha\ge 1$. To show the effects of this parameter, we run TATU+MOPO and TATU+TD3\_BC on hopper-medium-replay-v2 with $\alpha\in\{1.0,2.0,3.0,4.0,5.0\}$. We report the numerical comparison results in Table \ref{tab:coefficientalpha} and the learning curves in Figure \ref{fig:main_pesscoefhoppermediumreplay}. We see that larger $\alpha$ degrades the agent's performance of MOPO. This is harmful as the diversity in the imaginations is decreased and real data ratio $\eta$ adopted in MOPO is small, few samples are admitted. Too large $\alpha$ also causes more unstable in learning as shown in Figure \ref{fig:main_pesscoefhoppermediumreplay}. For TD3\_BC, however, a larger $\alpha$ is better as small $\alpha$ enables a large threshold $\epsilon$, i.e., bad transitions may be included and are more destructive for TD3\_BC (as it involves an imitation learning term) than for MOPO. We notice that we can luckily find a trade-off of $\alpha$ for different algorithms.

\begin{table}
\caption{Comparison of TATU+MOPO and TATU+TD3\_BC with different threshold coefficient $\alpha$ on hopper-medium-replay-v2. Each algorithm is run for 1M steps with 5 random seeds. We report the mean performance over the final 10 evaluations, in conjunction with the standard deviation.}
\label{tab:coefficientalpha}
\centering
\begin{tabular}{ccc}
\toprule
Coefficient $\alpha$   & TATU+MOPO    & TATU+TD3\_BC        \\ 
\midrule
1.0  &  101.4$\pm$0.8   &    32.1$\pm$11.1       \\
2.0  &  104.4$\pm$0.9    &    50.2$\pm$4.4          \\
3.0  &  98.7$\pm$3.3   &    52.2$\pm$10.1      \\
4.0  &  96.4$\pm$2.7   &     61.9$\pm$10.7      \\
5.0  &  52.8$\pm$31.7   &    58.6$\pm$13.4    \\
\bottomrule
\end{tabular}
\end{table}

\begin{figure}
    \centering
    \includegraphics[width=0.45\linewidth]{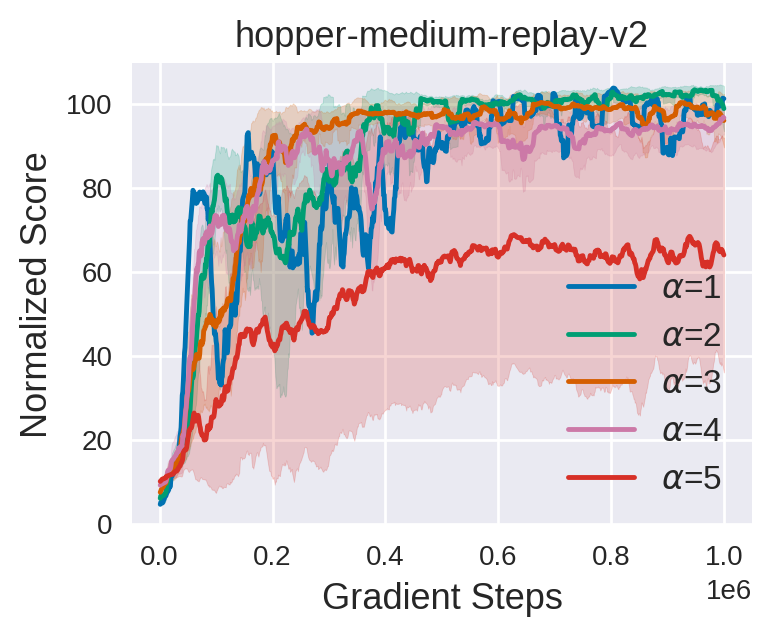}
    \includegraphics[width=0.45\linewidth]{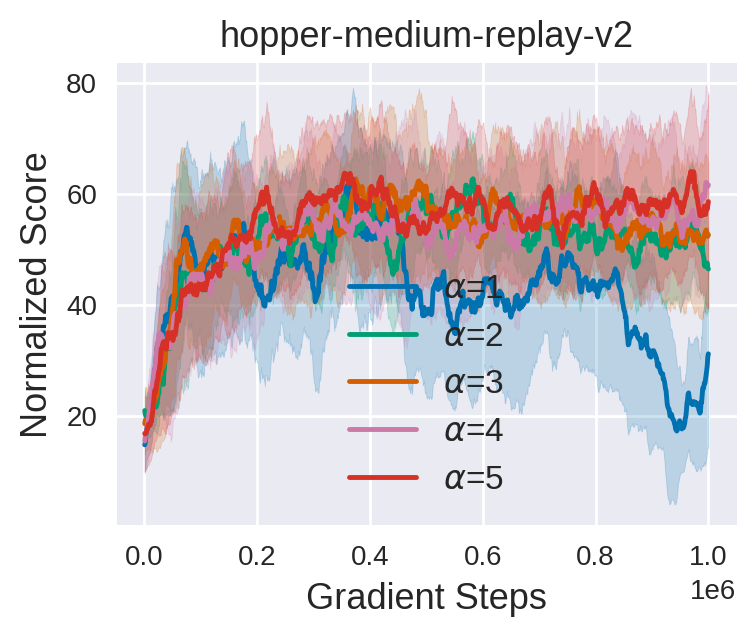}
    \caption{Normalized score comparison of TATU+MOPO (left) and TATU+TD3\_BC (right) on hopper-medium-replay-v2 under different threshold coefficients $\alpha$. The results are averaged over 5 different random seeds and the shaded region is the standard deviation.}
    \label{fig:main_pesscoefhoppermediumreplay}
\end{figure}

\noindent\textbf{Real data ratio $\eta$.} The real data ratio controls the proportion of real data in a sampled mini-batch, i.e., it determines how many synthetic samples are used for offline policy learning. $\eta$ also strongly depends on the specific dataset and setting. Model-based offline RL method usually uses a small $\eta$, and the $\eta$ for model-free offline RL algorithm relies heavily on the quality of the dataset, e.g., a small $\eta$ is preferred for datasets with poor quality, and a larger $\eta$ is better for expert datasets. To see how $\eta$ affects the performance of model-based and model-free offline RL algorithms with TATU, we run TATU+MOPO and TATU+TD3\_BC with different real data ratio $\eta\in\{0.05, 0.25, 0.5, 0.7, 0.9\}$ on hopper-medium-replay-v2. We summarize the results in Table \ref{tab:realdataratio}, where we find that TATU+MOPO achieves very good performance under a wide range of $\eta$, even a small $\eta=0.05$ thanks to the conservative trajectory truncation by TATU. As the dataset is of medium quality, TATU+TD3\_BC attains higher performance with larger $\eta$, as expected. The learning curves of TATU+MOPO and TATU+TD3\_BC under different real data ratios are shown in Figure \ref{fig:main_realdataratiohoppermediumreplay}, which are consistent with our analysis above. 

\begin{table}
\caption{Comparison of TATU+MOPO and TATU+TD3\_BC under different real data ratio $\eta$ on hopper-medium-replay-v2. The results are averaged over the final 10 evaluations and 5 random seeds. We report the mean performance and the standard deviation.}
\label{tab:realdataratio}
\centering
\begin{tabular}{ccc}
\toprule
Real data ratio $\eta$   & TATU+MOPO                  & TATU+TD3\_BC        \\ 
\midrule
0.05               &  104.4$\pm$0.9            & 9.1$\pm$4.4          \\
0.25               &  101.5$\pm$3.8            & 39.7$\pm$15.6         \\
0.5                &  101.1$\pm$2.7            & 65.4$\pm$1.2         \\
0.7                &  92.9$\pm$2.9            & 65.7$\pm$3.9         \\
0.9                &  85.1$\pm$12.4            & 68.8$\pm$17.9        \\
\bottomrule
\end{tabular}
\end{table}
\vspace{-0.3cm}

\begin{figure}
    \centering
    \includegraphics[width=0.45\linewidth]{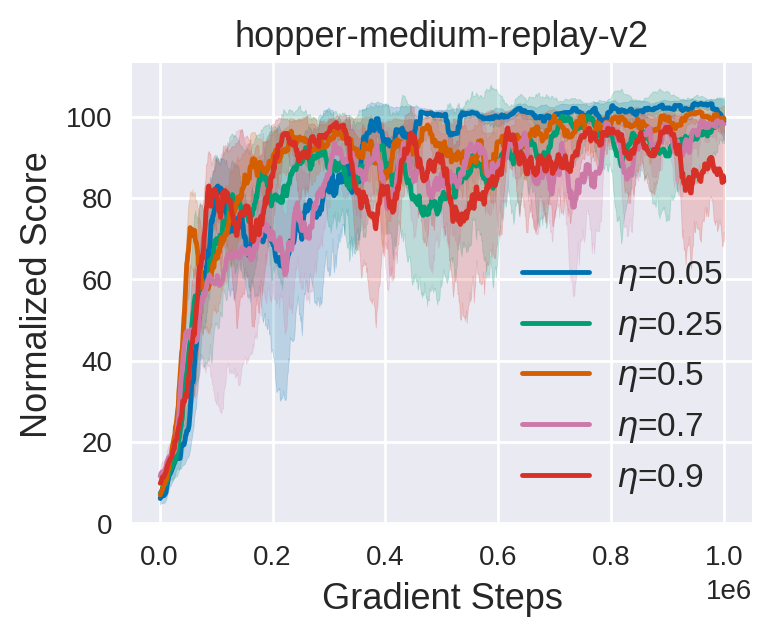}
    \includegraphics[width=0.45\linewidth]{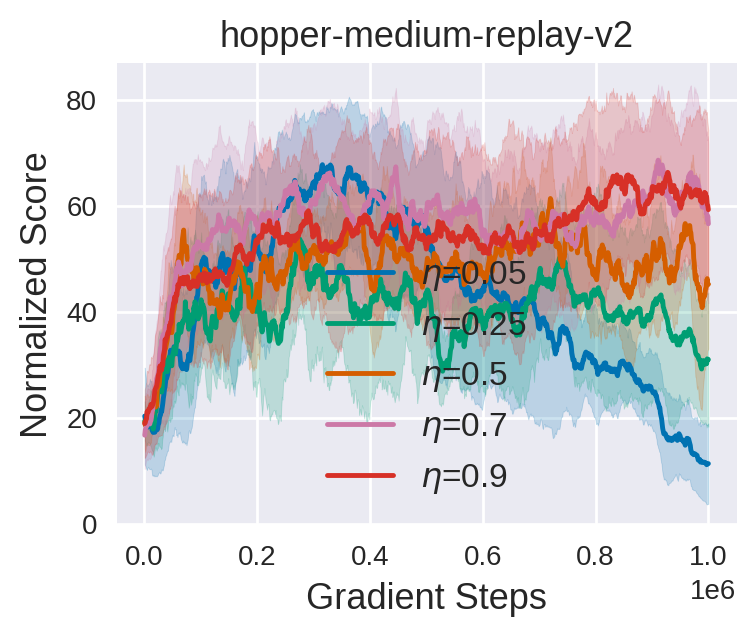}
    \caption{Performance of TATU+MOPO (left) and TATU+TD3\_BC (right) on hopper-medium-replay-v2 using different real data ratios. All methods are run for 5 different random seeds, and the shaded region captures the standard deviation.}
    \label{fig:main_realdataratiohoppermediumreplay}
\end{figure}

\subsection{Comparison with Other Relevant Methods}
\label{sec:compareothermethods}
In order to further show the advantages of our proposed TATU method, we compare it against other data selection methods, including CABI \cite{lyu2022double} and MOPP \cite{Zhan2021ModelBasedOP}. MOPP selects transitions by performing planning and then sorting the trajectory such that each sample in the left trajectory meets the uncertainty constraint. CABI trains bidirectional dynamics models and only admits transitions that the forward and backward model agree on. TATU ensures reliable data generation via adaptively truncating the imagined trajectory. We combine TATU and CABI with TD3\_BC and conduct experiments on 6 \textit{random}, \textit{medium-expert} datasets from D4RL MuJoCo datasets. We follow the guidance in the CABI paper and implement it on our own. We take the results of MOPP from its original paper directly. We run each algorithm for 1M gradient steps over 5 different random seeds. The results are presented in Table \ref{tab:otherdataselectionmethods}, where we report the average final performance. We find that TATU+TD3\_BC outperforms MOPP and CABI on most of the datasets, achieving the best performance on 4 out of 6 datasets. These we believe can well illustrate the superiority of TATU.

\vspace{-0.3cm}
\begin{table}
\caption{Comparison of TATU against CABI (with TD3\_BC as the base algorithm) and MOPP on six datasets from D4RL.}
\label{tab:otherdataselectionmethods}
\setlength\tabcolsep{3pt}
\centering
\begin{tabular}{@{}cccc@{}}
\toprule
Task Name                 & TATU+TD3\_BC      & CABI+TD3\_BC        &  MOPP \\ 
\midrule
half-r              & 12.1$\pm$2.3       &   \textbf{14.3}$\pm$0.4    &9.4$\pm$2.6 \\
hopper-r              & \textbf{31.6}$\pm$0.6    & 15.7$\pm$11.1    &13.7$\pm$2.5  \\
walker-r              & \textbf{21.4}$\pm$0.0    &  6.0$\pm$0.3      &6.3$\pm$0.1 \\
half-m-e          &  97.1$\pm$3.2     &  94.8$\pm$1.0      &\textbf{106.2}$\pm$5.1\\
hopper-m-e       & \textbf{113.0}$\pm$1.7      &    111.4$\pm$0.3    &95.4$\pm$28.0\\
walker2d-m-e        & \textbf{110.9}$\pm$0.5    &   110.5$\pm$0.6   &92.9$\pm$14.1\\
\bottomrule
\end{tabular}
\vspace{-0.6cm}
\end{table}

\section{Conclusion}
In this paper, we propose trajectory truncation with uncertainty (TATU) to facilitate both model-based and model-free offline RL algorithms. We adaptively truncate the imagined trajectory if the accumulated uncertainty of this trajectory is too large, which ensures the reliability of the synthetic samples. We theoretically demonstrate the advantages of our proposed truncation method. Empirical results on the D4RL benchmark show that TATU markedly improves the performance of model-based offline RL methods (e.g., MOPO) and model-free offline RL methods (e.g., CQL). These altogether illustrate the generality and effectiveness of TATU.

\ack This work was supported in part by the Science and Technology Innovation
2030-Key Project under Grant 2021ZD0201404.


\bibliography{ecai}

\clearpage

\clearpage
\appendix

\section{Missing Proofs}
\label{sec:missingproofs}

\subsection{Proof of Theorem \ref{theo:performancebound}}
In order to show Theorem \ref{theo:performancebound}, we first denote the pessimistic MDP $\mathcal{M}_p$ below:
\begin{definition}[Pessimistic MDP]
    \label{def:pessimisticmdp}
    The pessimistic MDP is defined by the tuple $\mathcal{M}_p = \langle\mathcal{S},\mathcal{A}, r_p, P_p, \rho_0,\gamma\rangle$, where $r_p (s,a)$ follows the Equation (\ref{eq:pmdpreward}) and $P_p(\cdot|s,a)$ is given by:
    \begin{equation}
        \label{eq:pessimisticmdptransition}
        P_p(\cdot|s_t,a_t) = \begin{cases}
            0, \qquad \qquad {\rm if}\, T_t^\epsilon(s_t,a_t) = 1, \\
            P(\cdot|s_t,a_t),\qquad {\rm otherwise},
        \end{cases}
    \end{equation}
    where $P(\cdot|s,a)$ is the transition probability of the true environmental dynamics.
\end{definition}

Then we introduce the following Lemma.
\begin{lemma}
\label{lemma:pessimisticperformance}
Set $\bar{r} = r_{\rm max}+\lambda u_{\rm max}+\kappa$. Given the pessimistic MDP $\mathcal{M}_p$ and the $\epsilon$-Pessimistic MDP $\hat{\mathcal{M}}_p$, the return difference of any policy $\pi$ on the two MDPs is bounded:
\begin{equation}
    \label{eq:pessimisticperformance}
    \begin{aligned}
    \left|J_{\rho_0} (\pi,\mathcal{M}_p) - J_{\hat{\rho}_0}(\pi,\hat{\mathcal{M}}_p) \right| &\le \dfrac{2\bar{r}}{1-\gamma} D_{\rm TV}(\rho_0, \hat{\rho}_0) \\
    &\quad + \bar{r} \cdot \epsilon.
    \end{aligned}
\end{equation}
\end{lemma}

\begin{proof}
By definition, it is easy to find that
\begin{align*}
    &\left|J_{\rho_0} (\pi,\mathcal{M}_p) - J_{\hat{\rho}_0}(\pi,\hat{\mathcal{M}}_p) \right| \\
    &= \left| \mathbb{E}_{\rho_0}\mathbb{E}_{P_p}\left[ \sum_{t=0}^\infty \gamma^t r_p(s_t,a_t) \right] - \mathbb{E}_{\hat{\rho}_0}\mathbb{E}_{\hat{P}_p}\left[ \sum_{t=0}^\infty \gamma^t r_p(s_t,a_t) \right] \right| \\
    &=\left| \mathbb{E}_{\rho_0}\mathbb{E}_{P_p}\left[ \sum_{t=0}^\infty \gamma^t r_p(s_t,a_t) \right] - \mathbb{E}_{\hat{\rho}_0}\mathbb{E}_{P_p}\left[ \sum_{t=0}^\infty \gamma^t r_p(s_t,a_t) \right] \right. \\
    &\left. + \mathbb{E}_{\hat{\rho}_0}\mathbb{E}_{P_p}\left[ \sum_{t=0}^\infty \gamma^t r_p(s_t,a_t) \right] - \mathbb{E}_{\hat{\rho}_0}\mathbb{E}_{\hat{P}_p}\left[ \sum_{t=0}^\infty \gamma^t r_p(s_t,a_t) \right] \right| \\
    &\le \underbrace{\left| \mathbb{E}_{\rho_0}\mathbb{E}_{P_p}\left[ \sum_{t=0}^\infty \gamma^t r_p(s_t,a_t) \right] - \mathbb{E}_{\hat{\rho}_0}\mathbb{E}_{P_p}\left[ \sum_{t=0}^\infty \gamma^t r_p(s_t,a_t) \right] \right|}_{L_1} \\
    & + \underbrace{\left|\mathbb{E}_{\hat{\rho}_0}\mathbb{E}_{P_p}\left[ \sum_{t=0}^\infty \gamma^t r_p(s_t,a_t) \right] - \mathbb{E}_{\hat{\rho}_0}\mathbb{E}_{\hat{P}_p}\left[ \sum_{t=0}^\infty \gamma^t r_p(s_t,a_t) \right] \right|}_{L_2}.
\end{align*}
It is easy to notice that $|r_p(s,a)|\le \bar{r} = r_{\rm max} + \lambda u_{\rm max} + \kappa$. Then for $L_1$, we have
\begin{align*}
    L_1 &= \left| \sum (\rho_0 - \hat{\rho}_0) \sum_t \sum_{a_t} P_p(\cdot|s_t,a_t) \gamma^t r_p(s_t,a_t) \right| \\
    &\le \bar{r} \left| \sum (\rho_0 - \hat{\rho}_0) \sum_t \sum_{a_t} P_p(\cdot|s_t,a_t) \gamma^t \right| \\
    &\le \bar{r} \left| \sum (\rho_0 - \hat{\rho}_0) \sum_t \gamma^t \right| \\
    &\le \dfrac{\bar{r}}{1-\gamma} \sum\left| \rho_0 - \hat{\rho}_0 \right| = \dfrac{2\bar{r}}{1-\gamma}D_{\rm TV}(\rho_0,\hat{\rho}_0).
\end{align*}
For $L_2$, we have
\begin{align*}
    L_2 &= \left|\mathbb{E}_{\hat{\rho}_0}\mathbb{E}_{P_p}\left[ \sum_{t=0}^\infty \gamma^t r_p(s_t,a_t) \right] - \mathbb{E}_{\hat{\rho}_0}\mathbb{E}_{\hat{P}_p}\left[ \sum_{t=0}^\infty \gamma^t r_p(s_t,a_t) \right] \right| \\
    &= \left| \sum \hat{\rho}_0 \sum_t \sum_{a_t} \left( P_p(\cdot|s_t,a_t) - \hat{P}_p(\cdot|s_t,a_t) \right) \gamma^t r_p(s_t,a_t) \right|.
\end{align*}
If the accumulated uncertainty is large, i.e., $T_t^\epsilon = 1$, then $P_p(\cdot|s,a) = \hat{P}_p(\cdot|s,a)=0$. Otherwise, we have $P_p(\cdot|s,a) = P(\cdot|s,a)$ and $\hat{P}_p(\cdot|s,a)=\hat{P}(\cdot|s,a)$, respectively. We denote the termination horizon as $T$. In the $\epsilon$-Pessimistic MDP, we require that $\sum_{t=0}^T \gamma^t D_{\rm TV}(P(\cdot|s_t,a_t), \hat{P}(\cdot|s_t,a_t))\le \epsilon$. Then it holds that for any horizon $i\le T, i\in\mathbb{Z}$, we have $|\sum_{t=0}^i\sum_{a_t}\gamma^t (P(\cdot|s_t,a_t) - \hat{P}(\cdot|s_t,a_t))|\le \epsilon$ almost surely. Therefore, we have
\begin{align*}
    L_2 &= \left| \sum \hat{\rho}_0 \sum_t \sum_{a_t} \left( P_p(\cdot|s_t,a_t) - \hat{P}_p(\cdot|s_t,a_t) \right) \gamma^t r_p(s_t,a_t) \right| \\
    &\le \bar{r} \left| \sum \hat{\rho}_0 \sum_t \sum_{a_t} \left( P_p(\cdot|s_t,a_t) - \hat{P}_p(\cdot|s_t,a_t) \right) \gamma^t \right| \\
    &\le \bar{r} \cdot\epsilon\cdot \left| \sum \hat{\rho}_0\right| 
    = \bar{r}\cdot\epsilon.
\end{align*}
By combing the upper bounds of $L_1$ and $L_2$, we have the desired result immediately.
\end{proof}

We also have the following lemma:
\begin{lemma}
    \label{lemma:differencetrueandfake}
    Set $\bar{r}=r_{\rm max} + \lambda u_{\rm max} + \kappa$. Given the pessimistic MDP $\mathcal{M}_p$ and the true MDP $\mathcal{M}$, the return difference of any policy $\pi$ on the two MDPs is bounded:
    \begin{equation}
        \label{eq:differencetrueandfake}
        \begin{aligned}
        &J_{\rho_0}(\pi,\mathcal{M}_p) \ge J_{\rho_0}(\pi,\mathcal{M}) - \dfrac{\bar{r}}{1-\gamma}, \\
        &J_{\rho_0}(\pi,\mathcal{M}_p) \le J_{\rho_0}(\pi,\mathcal{M}).
        \end{aligned}
    \end{equation}
\end{lemma}

\begin{proof}
    For the first part, we have
    \begin{align*}
        &J_{\rho_0}(\pi,\mathcal{M}_p) - J_{\rho_0}(\pi,\mathcal{M}) \\
        &= \sum \rho_0 \sum_t \sum_{a_t} \gamma^t \left(P_p(\cdot|s_t,a_t)r_p(s_t,a_t) - P(\cdot|s_t,a_t)r(s_t,a_t)\right) \\
        & = \sum \rho_0 \sum_t \sum_{a_t} \gamma^t \left(P_p(\cdot|s_t,a_t)(r_p(s_t,a_t) - r(s_t,a_t))\right) \\
        & \quad + \sum \rho_0 \sum_t \sum_{a_t} \gamma^t \left( (P_p(\cdot|s_t,a_t) - P(\cdot|s_t,a_t))r(s_t,a_t)\right).
    \end{align*}
    Furthermore, we have $r_p(s,a) - r(s,a)\ge -(\lambda u_{\rm max} + \kappa)$ as $0\le u(s,a)\le u_{\rm max},\forall\, s,a$. We also notice that $P_p(\cdot|s,a)=0$ if the trajectory is truncated at $(s,a)$, and $P_p(\cdot|s,a)=P(\cdot|s,a)$ otherwise. Hence, $P_p(\cdot|s,a) - P(\cdot|s,a)\le 0$. Then we have,
    \begin{align*}
        &J_{\rho_0}(\pi,\mathcal{M}_p) - J_{\rho_0}(\pi,\mathcal{M}) \\
        &\ge -(\lambda u_{\rm max}+\kappa)\sum \rho_0 \sum_t \sum_{a_t} \gamma^t P_p(\cdot|s_t,a_t) \\
        &\qquad - r_{\rm max} \sum \rho_0 \sum_t \sum_{a_t} \gamma^t P(\cdot|s_t,a_t)) = -\dfrac{\bar{r}}{1-\gamma}.
    \end{align*}
    For the second part, using the definition of the reward function in the pessimistic MDP, the rewards obtained by any policy $\pi$ on each transition in $\mathcal{M}_p$ is less than the reward obtained in $\mathcal{M}$. Therefore, the second part holds, which concludes the proof.
\end{proof}

We then can formally prove Theorem \ref{theo:performancebound}, which is restated below.
\begin{theorem}
\label{apptheo:performancebound}
Denote $\bar{r}=r_{\rm max}+\lambda u_{\rm max}+\kappa$. Then the return of any policy $\pi$ in the $\epsilon$-Pessimistic MDP $\hat{\mathcal{M}}_p$ and its original MDP $\mathcal{M}$ satisfies:
\begin{equation}
    \begin{aligned}
        J_{\hat{\rho}_0}(\pi,\hat{\mathcal{M}}_p) \ge &J_{\rho_0}(\pi,\mathcal{M}) - \dfrac{2\bar{r}}{1-\gamma}\cdot D_{\rm TV}(\rho_0,\hat{\rho}_0) \\
        & - \bar{r}\cdot \epsilon - \dfrac{\bar{r}}{1-\gamma},
    \end{aligned}
\end{equation}
\begin{equation}
    \begin{aligned}
    J_{\hat{\rho}_0}(\pi,\hat{\mathcal{M}}_p) \le &J_{\rho_0}(\pi,\mathcal{M}) + \dfrac{2\bar{r}}{1-\gamma}\cdot D_{\rm TV}(\rho_0,\hat{\rho}_0)\\
    & + \bar{r} \cdot \epsilon.
    \end{aligned}
\end{equation}
\end{theorem}

\begin{proof}
    We decompose $J_{\hat{\rho}_0}(\pi,\hat{\mathcal{M}}_p) - J_{\rho_0}(\pi,\mathcal{M})$ as:
    \begin{align*}
        \underbrace{J_{\hat{\rho}_0}(\pi,\hat{\mathcal{M}}_p) - J_{\rho_0}(\pi,\mathcal{M}_p)}_{\rm (I)} +  \underbrace{J_{\rho_0}(\pi,\mathcal{M}_p) - J_{\rho_0}(\pi,\mathcal{M})}_{\rm (II)}.
    \end{align*}
    Then by bounding term (I) with Lemma \ref{lemma:pessimisticperformance} and bounding term (II) using Lemma \ref{lemma:differencetrueandfake}, we can conclude the proof.
\end{proof}

\subsection{Proof of Corollary \ref{coro:suboptimal} and \ref{coro:simplified}}
We restate the Corollary \ref{coro:suboptimal} below.
\begin{corollary}
\label{appcoro:suboptimal}
If the policy in the $\epsilon$-Pessimistic MDP is $\delta_\pi$ sub-optimal, i.e., $J_{\hat{\rho}_0}(\pi,\hat{\mathcal{M}}_p)\ge J_{\hat{\rho}_0}(\pi^*,\hat{\mathcal{M}}_p) - \delta_\pi$, then we have
\begin{equation*}
    \begin{aligned}
    J_{\rho_0}(\pi^*,\mathcal{M}) - J_{\rho_0}(\pi,\mathcal{M}) \le &\delta_\pi + \dfrac{4\bar{r}}{1-\gamma}\cdot D_{\rm TV}(\rho_0,\hat{\rho}_0) \\
    & +2\bar{r} \epsilon + \dfrac{\bar{r}}{1-\gamma}.
\end{aligned}
\end{equation*}
\end{corollary}

\begin{proof}
    By using the results in Theorem \ref{theo:performancebound} and the assumption on the sub-optimality, we have
    \begin{align*}
        J_{\rho_0}(\pi, \mathcal{M}) &\ge J_{\hat{\rho}_0}(\pi,\hat{\mathcal{M}}_p) - \dfrac{2\bar{r}}{1-\gamma} D_{\rm TV}(\rho_0,\hat{\rho}_0) - \bar{r}\epsilon \\
        &\ge J_{\hat{\rho}_0}(\pi^*,\hat{\mathcal{M}}_p) - \delta_\pi - \dfrac{2\bar{r}}{1-\gamma} D_{\rm TV}(\rho_0,\hat{\rho}_0) - \bar{r}\epsilon \\
        & \ge J_{\rho_0}(\pi^*,\mathcal{M}) - \delta_\pi - \dfrac{4\bar{r}}{1-\gamma} D_{\rm TV}(\rho_0,\hat{\rho}_0) - 2\bar{r}\epsilon \\
        &\qquad \qquad - \dfrac{\bar{r}}{1-\gamma}.
    \end{align*}
    This concludes the proof.
\end{proof}

If the dataset is large enough, then all of the generated transitions will be in-distribution and hence admitted. Then $D_{\rm TV}(\rho_0,\hat{\rho}_0)$ and $\epsilon$ approach 0. By setting $D_{\rm TV}(\rho_0,\hat{\rho}_0)=0$ and $\epsilon=0$ in Corollary \ref{coro:suboptimal} conclude the proof of Corollary \ref{coro:simplified}.

\section{Experimental Setup and Hyperparameter Setup}
In this section, we present the detailed experimental setup for TATU. We introduce the datasets we adopted in this paper and provide a detailed hyperparameter setup we used for all of the base algorithms and our TATU.
\subsection{MuJoCo Datasets on the D4RL Benchmark}
In our evaluation, we utilize three tasks from the MuJoCo dataset, halfcheetah, hopper, and walker2d as illustrated in Figure \ref{fig:mujocodataset}. Each task contains five types of datasets: (1) \textbf{random} where samples are collected by a random policy; (2) \textbf{medium} where the samples are collected from an early-stopped SAC policy for 1M steps; (3) \textbf{medium-replay} where samples are gathered using the replay buffer of a policy that trained up to the medium-level agent; (4) \textbf{medium-expert} where samples are collected by mixing the medium-level data and expert-level data at a 50-50 ratio; (5) \textbf{expert} where the dataset is logged with an expert policy.
\begin{figure}
    \centering
    \includegraphics[width=0.31\linewidth]{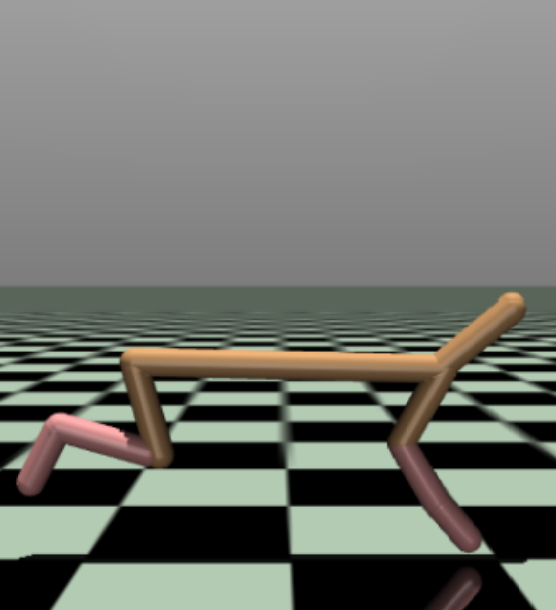}
    \includegraphics[width=0.29\linewidth]{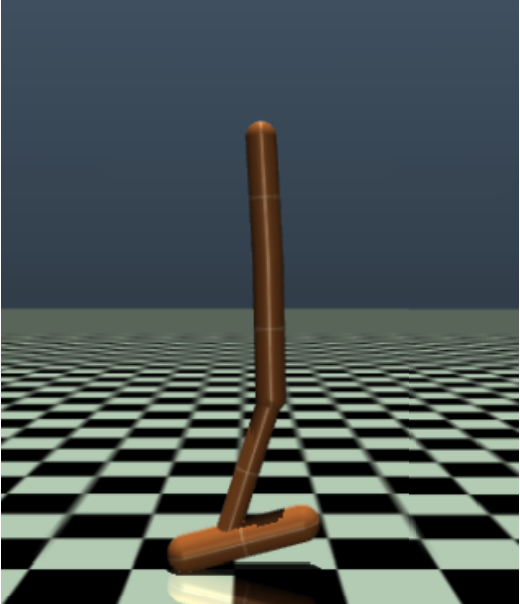}
    \includegraphics[width=0.29\linewidth]{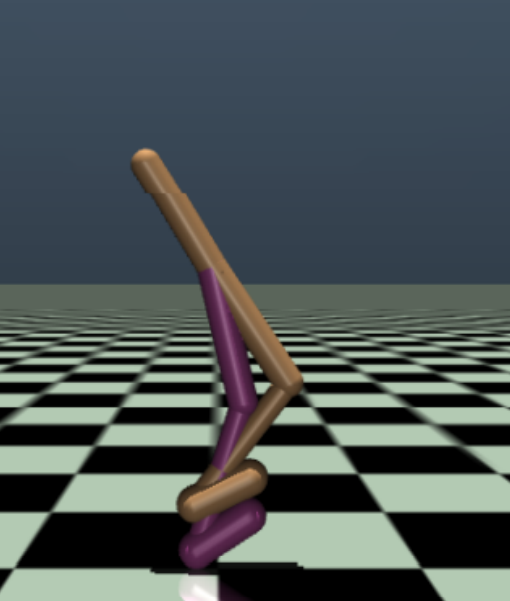}
    \caption{MuJoCo datasets from D4RL. From left to right, halfcheetah, hopper, walker2d.}
    \label{fig:mujocodataset}
\end{figure}

The normalized score is widely used for evaluating the performance of the agent on the offline dataset. Denote the expected return of a random policy as $J_{\rm random}$, and the expected return of an expert policy as $J_{\rm expert}$. Assume the return obtained by the policy $\pi$ gives $J_\pi$, then the normalized score (NS) gives:
\begin{equation}
    NS = \dfrac{J_\pi - J_{\rm random}}{J_{\rm expert} - J_{\rm random}}\times 100.
\end{equation}
It can be seen that the normalized score generally ranges from 0 to 100, where 0 corresponds to a random policy and 100 corresponds to an expert policy. The reference scores for the random policy and expert policy are previously set by D4RL.

\subsection{Implementation Details}
\label{app:implementation}
\textbf{TATU implementation details.} In our experiments, we train the forward dynamics model with a neural network parameterized by $\psi$ to output a multivariate Gaussian distribution for predicting the next state, 
\begin{equation*}
    \hat{p}_\psi(s^\prime|s,a) = \mathcal{N}(\mu_\theta(s,a),\Sigma_\phi(s,a)), \psi=\{\theta,\phi\},
\end{equation*} 
where $\mu_\theta(s,a)$ and $\Sigma_\phi(s,a)$ denote the mean and variance of the forward model, respectively. As mentioned in the main text, we model the difference in the current state and the next state to ensure local continuity. We model the probabilistic neural network with a multi-layer neural network that consists of 4 feedforward layers with 200 hidden units. Each intermediate layer in the network adopts swish activation. We train an ensemble of seven such neural networks by following prior work \cite{Janner2019WhenTT}. We use a validation set of 1000 transitions to validate the performance of the learned dynamics models. When generating imaginations, we first select the best five models out of seven candidates and then randomly pick one from them for imagined trajectory generation. For other parameters that are not quite relevant to TATU (e.g., hyperparameters of the base algorithms), we simply follow their default setting. We then give more experimental details on the key parameters in TATU as shown below:

\textbf{Reward penalty coefficient $\lambda$.}  $\lambda$ determines how large penalty we add to the generated samples, since the pessimistic reward function in the $\epsilon$-Pessimistic MDP gives $r_p(s,a) = r(s,a) - \lambda u(s_t,a_t)$. Larger $\lambda$ will inject more conservatism into the imagined samples. Considering the fact that there are many other important hyperparameters in TATU (e.g., threshold coefficient $\alpha$, real data ratio $\eta$, etc.), we choose to fix this parameter. To be specific, we set $\lambda=1.0$ on all of the evaluated tasks. One may possibly get better results by carefully tuning this parameter, and one can also refer to the suggested $\lambda$ adopted in the MOPO's official codebase (\href{https://github.com/tianheyu927/mopo}{https://github.com/tianheyu927/mopo}).

\textbf{Threshold coefficient $\alpha$.} The threshold coefficient $\alpha$ is one of the most important hyperparameters for TATU under different datasets and tasks. Note that the uncertainty threshold $\epsilon = \frac{1}{\alpha}\max_{i\in\{1,\ldots,|\mathcal{D}|\}}u(s_i,a_i)$. $\alpha$ thus controls how tolerable we are to the imagined trajectory. Larger $\alpha$ indicates that only very conservative trajectories that lie in the span of the dataset can be trusted, and smaller $\alpha$ will admit more trajectories. Generally, we set $\alpha\ge 1$ such that the accumulated uncertainty along the imagined trajectory will not exceed the maximum uncertainty in the real dataset, which ensures the reliability of the samples in the model buffer $\mathcal{D}_{\rm model}$. Empirically, we find that setting $\alpha=2$ can incur very satisfying performance on most of the evaluated datasets. Some exceptions are hopper-medium-v2, hopper-medium-expert-v2, and hopper-expert-v2 where $\alpha = 3.5$ is set to involve more pessimistic samples. 



\textbf{Rollout horizon $h$.} The rollout horizon decides how far can the synthetic trajectory branch from the starting state. Intuitively, a larger rollout horizon $h$ may involve more diverse transition data (as the trajectory is longer). In our main experiments, we set rollout horizon $h=5$ for all the datasets from MuJoCo, i.e., halfcheetah, hopper, and walker2d. In the main text, we conduct a parameter study with respect to $h$ in Section \ref{sec:parameterstudy} on hopper-medium-replay-v2. Furthermore, in this part, we conduct additional experiments using TATU+TD3\_BC on halfcheetah-expert-v2. The results are shown in Figure \ref{fig:horizonhalfcheetahexpert}. We find that TATU+TD3\_BC prefers a smaller horizon $h=1$ while it seems a larger horizon is better for TATU+MOPO and TATU+TD3\_BC on hopper-medium-replay-v2. This is due to the fact that generated samples are more likely to be OOD on expert datasets (with narrow span). On the medium-replay dataset, a longer horizon can still ensure that the imaginations are in-distribution and hence the base algorithms can benefit from more conservative imaginations (i.e., better generalization ability).


\begin{figure}
    \centering
    \includegraphics[width=0.80\linewidth]{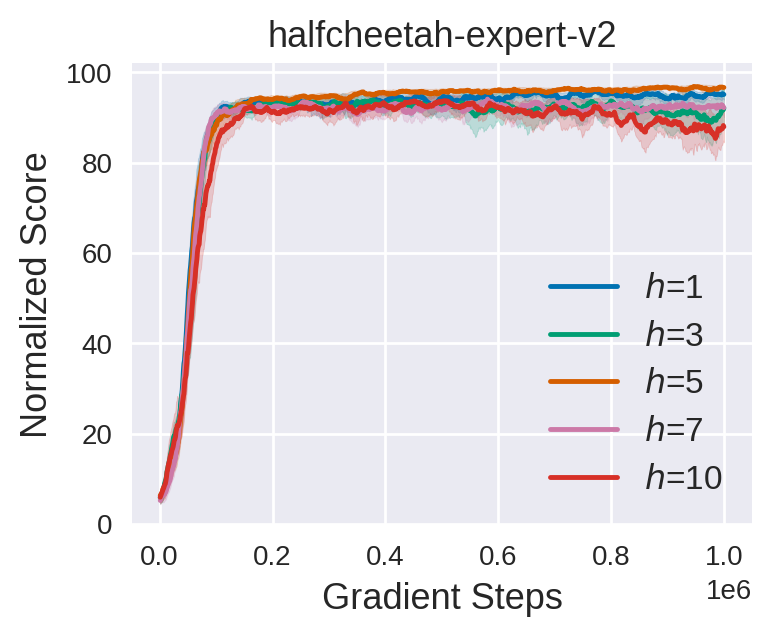}
    \caption{Performance of TATU+TD3\_BC on halfcheetah-expert-v2 with different $h$. The results are averaged over 5 runs. We report the mean performance and the standard deviation here.}
    \label{fig:horizonhalfcheetahexpert}
\end{figure}

\textbf{Real data ratio $\eta$.} When updating the offline policy, we sample a mini-batch (of size $B$) from $\mathcal{D}\cup\mathcal{D}_{\rm model}$. Then we sample $\eta B$ samples from the real dataset $\mathcal{D}$ and $(1-\eta)B$ samples from the model buffer $\mathcal{D}_{\rm model}$ given the real data ratio $\eta\in[0,1]$. When combining TATU with model-based offline RL methods, we follow their original parameter setup for real data ratio, i.e., $\eta=0.05$ for TATU+MOPO and $\eta=0.5$ for TATU+COMBO across all the datasets. As for the combination with model-free offline RL, $\eta$ strongly depends on the type of the dataset, i.e., small $\eta$ is better for datasets with poor quality while larger $\eta$ is preferred for datasets with high quality (e.g., expert datasets). We provide the detailed parameter setting for $\eta$ in Table \ref{apptab:realdataratio}. For most of the datasets, we use $\eta=0.7$ or $\eta=0.9$ which we find can ensure a quite satisfying performance. We show in Figure \ref{fig:main_realdataratiohoppermediumreplay} the learning curves of TATU+MOPO and TATU+TD3\_BC under different real data ratios $\eta\in\{0.05,0.25,0.5,0.7,0.9\}$. We further illustrate in Figure \ref{fig:real-ratio} the performance of TATU+TD3\_BC on halfcheetah-expert-v2 under different real data ratios. The result is consistent with our analysis above.


\begin{figure}
    \centering
    \includegraphics[width=0.80\linewidth]{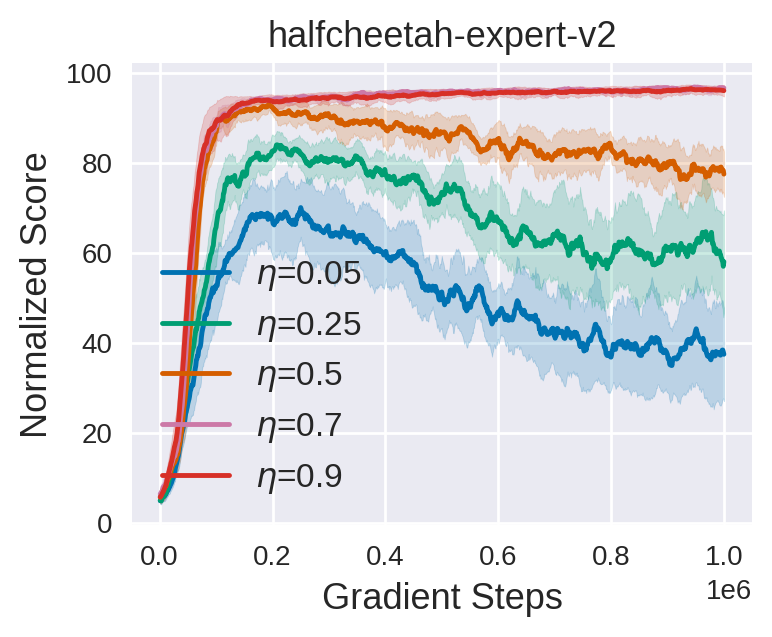}
    \caption{Performance of TATU+TD3\_BC on halfcheetah-expert-v2 with different $\eta$. The results are averaged over 5 runs and the shaded area denotes the standard deviation.}
    \label{fig:real-ratio}
\end{figure}

\begin{table}
\caption{Real data ratio $\eta$ adopted in TATU when combined with model-free offline RL algorithms. r = random, m = medium, m-r = medium-replay, m-e = medium-expert, e = expert. All experiments are run on D4RL MuJoCo ``v2" datasets.}
\label{apptab:realdataratio}
\setlength\tabcolsep{3pt}
\centering
\begin{tabular}{cccc}
\toprule
Task Name      &TATU+TD3\_BC  &TATU+CQL   &TATU+BCQ     \\ 
\midrule
half-r              & 0.7     & 0.7    & 0.7          \\
hopper-r            & 0.1     & 0.7    & 0.7                 \\
walker2d-r          & 0.1     & 0.7    & 0.7            \\
half-m-r            & 0.5     & 0.9    & 0.9            \\
hopper-m-r          & 0.7     & 0.7    & 0.7               \\
walker2d-m-r        & 0.5     & 0.9    & 0.9                   \\
half-m              & 0.7     & 0.9    & 0.9               \\
hopper-m            & 0.9     & 0.9    & 0.9               \\
walker2d-m          & 0.7     & 0.9    & 0.9               \\
half-m-e            & 0.7     & 0.9    & 0.9              \\
hopper-m-e          & 0.9     & 0.9    & 0.9                \\
walker2d-m-e        & 0.7     & 0.9    & 0.9               \\
half-e              & 0.7     & 0.9    & 0.9               \\
hopper-e            & 0.9     & 0.9    & 0.9               \\
walker2d-e          & 0.7     & 0.9    & 0.9               \\
\bottomrule
\end{tabular}
\end{table}

\section{Compute Infrastructure}
We list our compute infrastructure below.
\begin{itemize}
    \item Python: 3.7.13
    \item Pytorch: 1.12.0+cu113
    \item Gym: 0.24.1
    \item MuJoCo: 2.3.0
    \item D4RL: 1.1
    \item TensorFlow: 2.10.0
    \item GPU: RTX3090 ($\times 8$)
    \item CPU: AMD EPYC 7452
\end{itemize}

\section{Full Experimental Comparison}
\label{sec:fullcomparison}
Please refer to Table \ref{apptab:resultsonmodelbased} and \ref{apptab:resultsonmodelfree} the full experimental comparison of TATU against baseline methods when combined with model-based offline RL algorithms and model-free offline RL algorithms, respectively.

\begin{sidewaystable*}
\caption{Normalized average score comparison of TATU+MOPO and TATU+COMBO against their base algorithms and some recent baselines on the D4RL MuJoCo ``-v2" dataset. half = halfcheetah, r = random, m = medium, m-r = medium-replay, m-e = medium-expert. Each algorithm is run for 1M gradient steps with 5 different random seeds. We report the final performance. $\pm$ captures the standard deviation. We bold the top 2 score of the left part and the best score of the right part.}
\label{apptab:resultsonmodelbased}
\centering
\begin{tabular}{ccccc|cccccc}
\toprule
Task Name                 & TATU+MOPO   & MOPO      & TATU+COMBO   & COMBO & BC & CQL & IQL & DT & MOReL \\ 
\midrule
half-r        & 33.3 $\pm$ 2.6  & \textbf{35.9}$\pm$2.9 & 29.3 $\pm$ 2.7   & \textbf{38.8}$\pm$3.7 & 2.2$\pm$0.0 & 17.5$\pm$1.5 & 13.1$\pm$1.3 & - &  \textbf{38.9}$\pm$1.8 \\
hopper-r            & \textbf{31.3} $\pm$ 0.6  & 16.7$\pm$12.2 & \textbf{31.6} $\pm$ 0.6   & 17.9$\pm$1.4 & 3.7$\pm$0.6 & 7.9$\pm$0.4 & 7.9$\pm$0.2 & - & \textbf{38.1}$\pm$10.1 \\
walker2d-r           & \textbf{10.4} $\pm$ 0.7  & 4.2$\pm$5.7  &    5.3 $\pm$ 0.0          & \textbf{7.0}$\pm$3.6  & 1.3$\pm$0.1 & 5.1$\pm$1.3 & 5.4$\pm$1.2 & - &  \textbf{16.0}$\pm$7.7  \\
half-m-r & \textbf{67.2} $\pm$ 3.3  & \textbf{69.2}$\pm$1.1 & 57.8 $\pm$ 2.7   & 55.1$\pm$1.0 & 37.6$\pm$2.1 & \textbf{45.5}$\pm$0.7 & 44.2$\pm$1.2 & 36.6$\pm$0.8 & 44.5$\pm$5.6 \\
hopper-m-r      & \textbf{104.4} $\pm$ 0.9 & 32.7$\pm$9.4 & \textbf{100.7} $\pm$ 1.3   & 89.5$\pm$1.8 & 16.6$\pm$4.8 & 88.7$\pm$12.9 & \textbf{94.7}$\pm$8.6 & 82.7$\pm$7.0 & 81.8$\pm$17.0 \\
walker2d-m-r    & \textbf{75.3} $\pm$ 0.2  & 73.7$\pm$9.4 & \textbf{75.3} $\pm$ 1.7   & 56.0$\pm$8.6 & 20.3$\pm$9.8 & \textbf{81.8}$\pm$2.7 & 73.8$\pm$7.1 & 66.6$\pm$3.0 & 40.8$\pm$20.4 \\
half-m        & 61.9 $\pm$ 2.9  & \textbf{73.1}$\pm$2.4 & \textbf{69.2} $\pm$ 2.8   & 54.2$\pm$1.5 & 43.2$\pm$0.6 & 47.0$\pm$0.5 & 47.4$\pm$0.2 & 42.6$\pm$0.1 & \textbf{60.7}$\pm$4.4 \\
hopper-m             & \textbf{104.3} $\pm$ 1.3 & 38.3$\pm$34.9 & \textbf{100.0} $\pm$ 1.3  & 97.2$\pm$2.2 & 54.1$\pm$3.8 & 53.0$\pm$28.5 & 66.2$\pm$5.7 & 67.6$\pm$1.0 & \textbf{84.0}$\pm$17.0 \\
walker2d-m           & \textbf{77.9} $\pm$ 1.6  & 41.2$\pm$30.8 & 77.4 $\pm$ 0.9   & \textbf{81.9}$\pm$2.8  & 70.9$\pm$11.0 & 73.3$\pm$17.7 & \textbf{78.3}$\pm$8.7  & 74.0$\pm$1.4 & 72.8$\pm$11.9 \\
half-m-e & 74.1 $\pm$ 1.4  & 70.3$\pm$21.9 & \textbf{96.4} $\pm$ 3.6   & \textbf{90.0}$\pm$5.6  & 44.0$\pm$1.6 & 75.6$\pm$25.7 & 86.7$\pm$5.3 & \textbf{86.8}$\pm$1.3 & 80.4$\pm$11.7 \\
hopper-m-e      & \textbf{107.0} $\pm$ 1.3 & 60.6$\pm$32.5 & 106.5 $\pm$ 0.4 & \textbf{111.1}$\pm$2.9 & 53.9$\pm$4.7 & 105.6$\pm$12.9 & 91.5$\pm$14.3 & \textbf{107.6}$\pm$1.8 & 105.6$\pm$8.2 \\
walker2d-m-e    & \textbf{107.9} $\pm$ 0.9 & 77.4$\pm$27.9 & \textbf{114.6} $\pm$ 0.7  & 103.3$\pm$5.6 & 90.1$\pm$13.2 & 107.9$\pm$1.6 & \textbf{109.6}$\pm$1.0 & 108.1$\pm$0.2 & 107.5$\pm$5.6 \\
\midrule
Average score & \textbf{71.3} & 49.4 & \textbf{72.0} & 66.8 & 36.5 & 59.1 & 59.9 & - & \textbf{64.3} \\
\bottomrule
\end{tabular}
\end{sidewaystable*}

\begin{sidewaystable*}
\caption{Normalized average score comparison of TATU+TD3\_BC, TATU+CQL and TATU+BCQ against their base algorithms and some recent baselines on the D4RL MuJoCo ``-v2" dataset. half = halfcheetah, r = random, m = medium, m-r = medium-replay, m-e = medium-expert, e = expert. Each algorithm is run for 1M gradient steps across 5 different random seeds and the final mean performance is reported. $\pm$ captures the standard deviation. We bold the top 3 score of the left part and the best score of the right part.}
\label{apptab:resultsonmodelfree}
\centering
\begin{tabular}{ccccccc|ccc}
\toprule
Task Name                 & TATU+TD3\_BC             & TD3\_BC     & TATU+CQL                & CQL    &TATU+BCQ                 &BCQ  &  BC & IQL & DT \\
\midrule
half-r        & \textbf{12.1}$\pm$2.3  & 11.0$\pm$1.1        & \textbf{27.4}$\pm$2.6 & \textbf{17.5}$\pm$1.5   & 2.3$\pm$2.3             &2.2$\pm$0.0 & 2.2$\pm$0.0 & \textbf{13.1}$\pm$1.3 & -   \\
hopper-r             & \textbf{31.6}$\pm$0.6  & 8.5$\pm$0.6         & \textbf{32.3}$\pm$0.7 & 7.9$\pm$0.4    & \textbf{10.3}$\pm$0.8   &7.8$\pm$0.6 & 3.7$\pm$0.6 & \textbf{7.9}$\pm$0.2 & -  \\
walker2d-r           & \textbf{21.4}$\pm$0.0  & 1.6$\pm$1.7         & \textbf{23.0}$\pm$0.0 & \textbf{5.1}$\pm$1.3   & 3.4$\pm$0.4              &4.9$\pm$0.1 & 1.3$\pm$0.1 & \textbf{5.4}$\pm$1.2 & -   \\
half-m-r & \textbf{45.9}$\pm$0.6  & 44.6$\pm$0.5                & \textbf{48.0}$\pm$0.7         & \textbf{45.5}$\pm$0.7  & 43.5$\pm$0.3    &42.2$\pm$0.9 & 37.6$\pm$2.1 & \textbf{44.2}$\pm$1.2 & 36.6$\pm$0.8  \\
hopper-m-r      & 65.7$\pm$3.9  & 60.9$\pm$18.8        & \textbf{96.8}$\pm$2.6 & \textbf{88.7}$\pm$12.9  & \textbf{72.9}$\pm$0.4   &60.9$\pm$14.7 & 16.6$\pm$4.8 & \textbf{94.7}$\pm$8.6 & 82.7$\pm$7.0  \\
walker2d-m-r    & \textbf{81.9}$\pm$2.7  & \textbf{81.8}$\pm$5.5        &\textbf{85.5}$\pm$1.2 & \textbf{81.8}$\pm$2.7   & 77.7$\pm$1.0   & 57.0$\pm$9.6 & 20.3$\pm$9.8 & \textbf{73.8}$\pm$7.1 & 66.6$\pm$3.0  \\
half-m        & \textbf{48.4}$\pm$2.7  & \textbf{48.3}$\pm$0.3          & 44.9$\pm$0.3                & 47.0$\pm$0.5  & \textbf{47.6}$\pm$0.2    &46.6$\pm$0.4 & 43.2$\pm$0.6 & \textbf{47.4}$\pm$0.2 & 42.6$\pm$0.1  \\
hopper-m             & \textbf{62.0}$\pm$1.0  & 59.3$\pm$4.2        & \textbf{68.9}$\pm$0.9 & 53.0$\pm$28.5  & \textbf{71.0}$\pm$0.3   &59.4$\pm$8.3 & 54.1$\pm$3.8 & 66.2$\pm$5.7 & \textbf{67.6}$\pm$1.0 \\
walker2d-m           & \textbf{84.3}$\pm$0.2  & \textbf{83.7}$\pm$2.1        &  65.7$\pm$0.5  & 73.3$\pm$17.7 & \textbf{80.5}$\pm$0.4  &71.8$\pm$7.2 & 70.9$\pm$11.0 & \textbf{78.3}$\pm$8.7 & 74.0$\pm$1.4 \\
half-m-e & \textbf{97.1}$\pm$3.2  & 90.7$\pm$4.3        & 78.9$\pm$0.5 & 75.6$\pm$25.7   & \textbf{96.1}$\pm$0.2   &\textbf{95.4}$\pm$2.0  & 44.0$\pm$1.6 & 86.7$\pm$5.3 & \textbf{86.8}$\pm$1.3 \\
hopper-m-e      & \textbf{113.0}$\pm$1.7 & 98.0$\pm$9.4        & \textbf{111.5}$\pm$1.0 & 105.6$\pm$12.9 & \textbf{108.2}$\pm$0.3 &106.9$\pm$5.0 & 53.9$\pm$4.7 & 91.5$\pm$14.3 & \textbf{107.6}$\pm$1.8 \\
walker2d-m-e    & \textbf{110.9}$\pm$0.5 & 110.1$\pm$0.5       & \textbf{110.2}$\pm$0.1  & 107.9$\pm$1.6 & \textbf{111.7}$\pm$0.3 &107.7$\pm$3.8 & 90.1$\pm$13.2 & \textbf{109.6}$\pm$1.0 & 108.1$\pm$0.2 \\
half-e        & \textbf{97.4}$\pm$0.4  & \textbf{96.7}$\pm$1.1        & 90.8$\pm$3.4   &  \textbf{96.3}$\pm$1.3 & \textbf{96.3}$\pm$1.4  &89.9$\pm$9.6  & 91.8$\pm$1.5 & \textbf{95.0}$\pm$0.5 & -  \\
hopper-e             & \textbf{111.8}$\pm$1.1 & \textbf{107.8}$\pm$7       & 106.8$\pm$0.9  & 96.5$\pm$28.0 & 103.9$\pm$0.6         &\textbf{109.0}$\pm$4.0 & 107.7$\pm$0.7 & \textbf{109.4}$\pm$0.5 & - \\
walker2d-e           & \textbf{110.0}$\pm$0.1         & \textbf{110.2}$\pm$0.3 & 108.3$\pm$0.1 & 108.5$\pm$0.5 & \textbf{109.7}$\pm$0.3 &106.3$\pm$5.0 & 106.7$\pm$0.2 &  \textbf{109.9}$\pm$1.2 & - \\
\midrule
Average score & \textbf{72.9} & 67.5 & \textbf{73.3} & 67.3 & \textbf{69.0} & 64.5 & 49.6 & \textbf{68.9} & - \\
\bottomrule
\end{tabular}
\end{sidewaystable*}

\end{document}